\documentclass[twoside]{article}

\usepackage[accepted]{aistats2019}

\usepackage[utf8]{inputenc} %
\usepackage[T1]{fontenc}    %
\usepackage[usenames,dvipsnames]{color}
\definecolor{RefColor}{rgb}{0,0,.85}
\usepackage[colorlinks,linkcolor=RefColor,citecolor=RefColor,urlcolor=RefColor]{hyperref}       %
\usepackage{url}            %
\usepackage{booktabs}       %
\usepackage{amsfonts}       %
\usepackage{nicefrac}       %
\usepackage{microtype}      %

\usepackage{euscript,microtype} %
\usepackage[labelfont=bf,font=footnotesize,width=.9\textwidth]{caption}
\usepackage{subcaption,siunitx,wrapfig}
\usepackage[table]{xcolor}
\usepackage{amsmath, amssymb, amsthm, thmtools, mathtools}
\usepackage[%
minnames=1,maxnames=99,maxcitenames=3, style=numeric, doi=false,url=false,
firstinits=true,hyperref,natbib,backend=bibtex,sorting=nyt]{biblatex}%
\usepackage{enumitem}

\newcommand{\defnphrase}[1]{\emph{#1}}

\newcommand{\defeq}{\coloneqq}
\newcommand{\Reals}{\mathbb{R}}
\newcommand{\Nats}{\mathbb{N}}
\newcommand{\NNReals}{\Reals_{+}}
\newcommand{\edges}{e}
\newcommand{\vertices}{v}

\newcommand{\EE}{\mathbb{E}}
\renewcommand{\Pr}{\mathbb{P}}
\newcommand{\convPr}{\xrightarrow{\,p\,}}

\newcommand{\given}{\mid}

\newcommand{\uniDist}{\mathrm{Uni}}
\newcommand{\poiDist}{\mathrm{Poi}}
\newcommand{\as}{\textrm{ a.s.}}

\newcommand{\abs}[1]{\left\lvert#1 \right\rvert}

\newcommand{\dist}{\ \sim\ }
\newcommand{\distiid}{\overset{\mathrm{iid}}{\dist}}
\newcommand{\distind}{\overset{ind}{\dist}}

\newcommand{\PP}{\Pi}

\newcommand{\Lebesgue}{\Lambda}

\global\long\def\iid{i.i.d.\ }

\global\long\def\bern{\mathrm{Bern}}

\global\long\def\bernDist{\bern}

\global\long\def\uniDist{\mathrm{Uni}}

\global\long\def\poiDist{\mathrm{Poi}}

\global\long\def\given{\mid}

\global\long\def\distiid{\overset{iid}{\dist}}

\global\long\def\distind{\overset{ind}{\dist}}

\global\long\def\Reals{\mathbb{R}}

\global\long\def\Nats{\mathbb{N}}

\global\long\def\NNReals{\Reals_{+}}

\global\long\def\as{\textrm{ a.s.}}

\global\long\def\grad{\nabla}

\providecommand\given{} %
\newcommand\SetSymbol[1][]{
  \nonscript\,#1:\nonscript\,\mathopen{}\allowbreak}
\DeclarePairedDelimiterX\Set[1]{\lbrace}{\rbrace}%
{ \renewcommand\given{\SetSymbol[]} #1 }

\newcommand{\samp}{\mathsf{Smpl}}

\DeclareMathOperator*{\argmin}{argmin}

\usepackage[capitalize]{cleveref}

\crefname{lemma}{Lemma}{Lemmas}
\crefname{cor}{Corollary}{Corollaries}
\crefname{theorem}{Theorem}{Theorems}
\crefname{assumption}{Assumption}{Assumptions}

\declaretheorem[style=plain,numberwithin=section,name=Theorem]{theorem}
\declaretheorem[style=plain,sibling=theorem,name=Lemma]{lemma}

\declaretheorem[style=definition,sibling=theorem,name=Definition]{defn}

\declaretheorem[style=remark,sibling=theorem,name=Remark]{remark}

\usepackage{xcolor}
\newcommand{\LATER}[1]{\error}
\newcommand{\fLATER}[1]{\error}
\newcommand{\TBD}[1]{\error}
\newcommand{\fTBD}[1]{}
\newcommand{\PROBLEM}[1]{\error}
\newcommand{\fPROBLEM}[1]{\error}
\newcommand{\NA}[1]{#1} 

\declaretheoremstyle[
    spacebelow=\parsep,
    spaceabove=\parsep,
  mdframed={
    backgroundcolor=gray!10!white,     %
    hidealllines=true, 
    innertopmargin=8pt, 
    innerbottommargin=4pt, 
    skipabove=8pt,
    skipbelow=10pt,
    nobreak=true
}
]{grayboxed} 
\declaretheorem[style=plain]{auxtheorem}
\declaretheorem[style=grayboxed,sibling=auxtheorem]{algorithm}
\declaretheorem[style=grayboxed,name=Algorithm]{nalgorithm}

\usepackage{dsfont}
\usepackage{tikz}

\newcommand{\kword}[1]{\emph{#1}}

\newcommand{\tspace}{\mathcal{T}}

\newcommand{\argdot}{\,\vcenter{\hbox{\tiny$\bullet$}}\,}

\newcommand{\sample}{\mathbb{S}}

\def\bX{\overline{X}}

\def\bsample{\overline{\sample}}

\newcommand{\genparam}{\theta}
\newcommand{\globparam}{\gamma}
\newcommand{\embedding}{\lambda}

\runningtitle{Relational ERM}

\runningauthor{Veitch, Austern, Zhou, Blei, Orbanz}

\begin{document}
\newcommand*\samethanks[1][\value{footnote}]{\footnotemark[#1]}
\twocolumn[

\aistatstitle{Empirical Risk Minimization and \\Stochastic Gradient Descent for Relational Data}
\runningtitle{Relational ERM}

\aistatsauthor{Vicor Veitch\footnotemark \And Morgane Austern\samethanks \And Wenda Zhou\samethanks \And David M. Blei \And Peter Orbanz}
\aistatsaddress{Columbia University}
]
\footnotetext{Equal contribution}

\begin{abstract}
  Empirical risk minimization is the main tool for prediction problems, but
  its extension to relational data remains unsolved.
  We solve this problem using recent ideas from graph sampling theory to
  (i) define an empirical risk for relational data and
  (ii) obtain stochastic gradients for this empirical risk that are automatically unbiased.
  This is achieved by considering the method by which data is sampled from a graph
  as an explicit component of model design. 
  By integrating fast implementations of graph sampling schemes with standard automatic differentiation tools,
  we provide an efficient turnkey solver for the risk minimization problem.
  We establish basic theoretical properties of the procedure.
  Finally, we demonstrate relational ERM with application to two non-standard problems: 
  one-stage training for semi-supervised node classification,
  and learning embedding vectors for vertex attributes.
  Experiments confirm that the turnkey inference procedure is effective in practice, 
  and that the sampling scheme used for model specification has a strong effect on model performance.
  Code is available at \href{https://github.com/wooden-spoon/relational-ERM}{github.com/wooden-spoon/relational-ERM}.
   
\end{abstract}

\section{Introduction}

\def\samp{\mathsf{Sample}}

Relational data is data that can be represented as a graph,
possibly annotated with additional information.
An example is the link graph of a social network,
annotated by user profiles.
We consider prediction problems for such data.
For example, how to predict the preferences of a user of a social network
using both the preferences and profiles of other users,
and the network itself?
In the classical case of \iid sequence data---where the observed data does not include link structure---the data 
decomposes into individual examples.
Prediction methods for such data typically rely on this decomposition, 
e.g., predicting a user's preferences from only the profile of the user, ignoring the network structure.
Relational data, however, does not decompose; 
e.g., because of the link structure, a social network can not be decomposed into individual users.
Accordingly, classical methods do not generally apply to relational data,
and new methods cannot be developed with the same ease as for \iid sequence data.

With \iid sequence data, prediction problems are
typically solved with models fit by empirical risk minimization (ERM) \cite{Vapnik:1992,Vapnik:1995,Shalev-Shwartz:Ben-David:2014}.
We give an (unusual) presentation of ERM that anticipates the relational case.
The observed data is a set $\bsample_n = \{\bX_1,\ldots,\bX_n\}$ that decomposes into examples $\bX_i=(X_i,Y_i)$.
The task is to choose a predictor $\pi$ that completes $X$ by estimating missing information $Y$, e.g., a class label.
An ERM model is defined by two parts: (i) a hypothesis class ${\lbrace\pi_{\genparam}|\genparam\in\tspace\rbrace}$ from which $\pi$ is chosen, %
and (ii) a loss function $L$ where $L(\bar{x}; \genparam) \in \NNReals$ 
measures the reconstruction error of predictor $\pi_\genparam$ on example $\bar{x}$.   
\NA{The empirical risk is the expected loss on an example randomly selected from the dataset:}
\begin{equation}\label{iid:risk}
  \hat{R}(\theta, \bsample_n) \defeq \EE_{\bX\sim \mathbb{F}(\bsample_n)}[L(\bX; \genparam)|\bsample_n],
\end{equation}
where $\mathbb{F}(\bsample_n)$ is the empirical distribution.\footnote{The empirical risk is more often equivalently written as $\hat{R}(\theta, \bsample_n) = \frac{1}{n}\sum_{i\leq n}L(\bX_i;\genparam)$.}
The ERM dogma is to select the predictor $\pi_{\hat{\genparam}_n}$ given by $\hat{\genparam}_n = \argmin_\genparam \hat{R}(\genparam, \bsample_n)$.
That is, the objective function that defines learning is the empirical risk. 

ERM has two useful properties. (1) It provides a principled framework for defining new machine learning methods. 
In particular, when examples are generated i.i.d., model-agnostic results guarantee that ERM models cohere as more data is collected (e.g., in the sense of statistical convergence) \cite{Shalev-Shwartz:Ben-David:2014}. 
(2) For differentiable models, mini-batch stochastic gradient descent (SGD) can 
efficiently solve the minimization problem (albeit, approximately).
The ease of SGD comes from the definition of the empirical risk as the expectation over a randomly subsampled example: 
the gradient of the loss on a randomly subsampled example is an unbiased estimate of the gradient of the empirical risk. 
\NA{Combined with automatic differentiation, this provides a turnkey approach to fitting machine-learning models.}

Returning to relational data,
the observed data is now a graph $\overline{G}_n$ of size $n$ (e.g., the number of vertices or edges).
The graph is possibly annotated, e.g., by vertex labels. 
We further consider $G_n$ as an incomplete version of $\overline{G}_n$.
For example, $G_n$ may censor labels of the vertices or some of the edges from $\overline{G}_n$.
In relational learning, the task is to find a predictor $\pi$ that completes $G_n$ by estimating the missing information.
Typically, $\pi$ is chosen from a parameterized family ${\lbrace\pi_{\genparam}|\genparam\in\tspace\rbrace}$
to minimize an objective function $\mathcal{O}_n(\genparam, \overline{G}_n)$.
Unlike the empirical risk, the objective $\mathcal{O}_n$ is not built from a loss on individual examples; $\mathcal{O}_n$ must be specified for the entire observed graph.

In relational learning, there is not yet a framework 
that has properties (1) and (2) of ERM.
The challenge is that relational data does not decompose into individual examples.
Regarding (1), theory is elusive because the \iid sequence assumption is meaningless for relational data.
This makes it difficult to reason about what happens as more data is collected.
Regarding (2), mini-batch SGD is not generally applicable even for differentiable models. 
SGD requires unbiased estimates of the full gradient.
For a random subgraph $G_k$ of $G_n$, 
the stochastic gradient $\grad_\genparam \mathcal{O}_k(\pi_\genparam(G_k), \overline{G}_k)$ is not generally unbiased.
In particular, the bias depends on the choice of random sampling scheme used to select the subgraph.
Circumventing these two issues requires either careful design of the objective function used for learning \citep[e.g.,][]{Perozzi:Al-Rfou:Skiena:2014,Grover:Leskovec:2016,Chamberlain:Clough:Deisenroth:2017,Yang:Cohen:Salakhudinov:2016,Hamilton:Ying:Leskovec:2017:inductive}, or
model-specific derivation and analysis. 
For example, graph convolutional networks \cite{Kipf:Welling:2016,Kipf:Welling:2017,Schlichtkrull:Kipf:Bloem:vandenBerg:Titov:Welling:2018,vandenBerg:Kipf:Welling:2017} use full batch gradients, and scaling training requires custom derivation of stochastic gradients \cite{Chen:Zhu:Song:2018}.

This paper introduces relational ERM, a generalization of ERM to relational data.
Relational ERM provides a recipe for machine learning with relational data that preserves the two important properties of ERM:
\begin{enumerate}
\item It provides a simple way to define (task-specific) relational learning methods, and
\item For differentiable models, relational ERM minimization can be efficiently solved 
in a turnkey fashion by mini-batch stochastic gradient descent. 
\end{enumerate}
Relational ERM mitigates the need for model-specific analysis and fitting procedures.

\NA{
Extending turnkey mini-batch SGD to relational data allows the easy use of autodiff-based machine-learning frameworks for relational learning.
To facilitate this, we provide fast implementations of a number of graph subsampling algorithms, and integration with TensorFlow.}\footnote{\href{https://github.com/wooden-spoon/relational-ERM}{github.com/wooden-spoon/relational-ERM}}

In \cref{sec:rerm} we define relational ERM models and 
show how to automatically calculate unbiased mini-batch stochastic gradients.
In \cref{sec:example_models} we explain connections to previous work on machine learning for graph data
and we illustrate how to develop task-specific relational ERM models. 
In \cref{sec:samp-algorithms} we review several randomized algorithms for subsampling graphs. 
Relational ERM models require the specification of such algorithms.
In \cref{sec:theory} we establish theory for relational ERM models.
The main insights are: (i) the \iid assumption can be replaced by an assumption on how the data is collected \cite{Orbanz:2017,Veitch:Roy:2016,Borgs:Chayes:Cohn:Veitch:2017,Crane:Dempsey:2016:snm}, and, (ii) the choice of randomized sampling algorithm is necessarily viewed as a model component.
In \cref{sec:Experiments}, we study relational ERM empirically by implementing the models of \cref{sec:example_models}.
We observe that the turnkey mini-batch SGD procedure succeeds in efficiently fitting the models, and that the choice of graph subsampling algorithm has a large effect in practice.  

\section{Relational ERM and SGD}\label{sec:rerm}

Our aim is to define relational ERM in analogy with classical ERM. 
The fundamental challenge is that relational data does not decompose into individual examples.
Classical ERM uses the empirical distribution to define the objective function \cref{iid:risk}.
There is no canonical analogue of the empirical distribution for relational data.

The first insight is that the empirical distribution may be viewed as a randomized algorithm for subsampling the dataset.
The required analogue is then a randomized algorithm for subsampling a graph. 
In the \iid setting, uniform subsampling is almost always used. 
However, there are many possible ways to sample from a graph.
We review a number of possibilities in \cref{sec:samp-algorithms}.
For example, the sampling algorithm might draw a subgraph induced by sampling $k$ vertices at random, 
or the subgraph induced by a random walk of length $k$.
The challenge is that there is no a priori criterion for deciding which sampling algorithm is ``best.''

Our approach is to give up and declare victory:
we \emph{define} the required analogue as a \emph{component of model design}.  
We require the analyst to choose a randomized sampling algorithm $\samp$, where $\samp(\overline{G}_n,k)$ is a random subgraph of size $k$. 
The choice of $\samp$ defines a notion of ``example.'' This allows us to complete the analogy to classical ERM.

A \kword{relational ERM model} is defined by three ingredients:
\begin{enumerate}
\item A sampling routine $\samp$.
\item A predictor class ${\lbrace \pi_\genparam |\genparam\in\tspace\rbrace}$ with parameter $\genparam$. %
\item A loss function $L$, where $L(\overline{G}_k; \genparam)$ measures the reconstruction quality of $\pi_\genparam$ on example $G_k$.
\end{enumerate}
The objective function is defined in analogy with the empirical risk \cref{iid:risk}.
The \kword{relational empirical risk} is:
\begin{equation}
  \label{relational:risk}
  \hat{R}_k(\pi, \overline{G}_n) := \EE_{\overline{G}_k=\samp(\overline{G}_n, k)}[L(\overline{G}_k;\genparam) \given \overline{G}_n].
\end{equation}
\kword{Relational empirical risk minimization} selects a predictor $\hat{\pi}$ that minimizes the relational empirical risk,
\begin{equation}\label{eqn:rerm}
  \hat{\pi}:=\pi_{\hat{\genparam}_n}
  \quad\text{ where }\quad
  \hat{\genparam}_n \; := \; \argmin_{\genparam} \hat{R}_k(\pi_\genparam, \overline{G}_n) \;.
\end{equation} 

\subsection*{Stochastic gradient descent}
\label{sec:sgd}
A crucial property of relational ERM is that SGD can be applied to 
solve the minimization problem \cref{eqn:rerm} without any model specific analysis.
Define a stochastic gradient as $\grad_{\genparam}L(\samp(G_{n},k); \genparam)$, 
the gradient of the loss computed on a sample of size $k$ drawn with $\samp$.
Observe that
\begin{align*}%
\grad_{\theta}\hat{R}_{r}(\theta,G_n)
&\;=\;
\grad_{\theta}\EE[L(\samp(G_n,k); \theta)\given \overline{G}_n] \\
&\;=\;
\EE[\grad_{\theta}L(\samp(G_n,k); \theta)\given \overline{G}_n].
\end{align*}
That is, the random gradient $\grad_{\genparam}L(\samp(G_{n},k); \genparam)$ is an unbiased estimator of the gradient of the full relational empirical risk. 
If $\samp$ is computationally efficient, then SGD with this stochastic estimator can solve the relational ERM.

To specify a relational ERM model in practice, the practitioner implements the three ingredients in code.
Machine-learning frameworks provide tools to make it easy to specify a class of predictors and a per-example loss function,
which are ingredients of classical ERM.
Relational ERM additionally requires implementing $\samp$ and integrating it with a machine-learning framework. 
In practice, $\samp$ can be chosen from a standard library of sampling routines. 
To that end, we provide efficient implementations of a number of routines and integration with an automatic differentiation framework (TensorFlow).\footnote{\href{https://github.com/wooden-spoon/relational-ERM}{github.com/wooden-spoon/relational-ERM}}
This gives an effective ``plug-and-play'' approach for defining and fitting models.

\section{Example Models}
\label{sec:example_models}

We consider several examples of relational ERM models.
We split the
parameter into a pair $\genparam=(\globparam,\embedding)$:
the global parameters $\globparam$ are shared across the entire graph, 
and the embedding parameters $\embedding$ provide low-dimensional embeddings $\embedding_v$ for each vertex $v$.
Informally, global parameters encode population 
properties---``people with different political affiliation are less likely to be friends''---and
the embeddings encode per-vertex information---``Bob is a radical vegan.''

\subsubsection*{Graph representation learning}
Methods for learning embeddings of vertices are widely studied; see \cite{Hamilton:Ying:Leskovec:2017:review} for a review. 
Many such methods rely on decomposing the graph into neighborhoods determined by (random) walks of fixed size.
One example is Node2Vec \cite{Grover:Leskovec:2016} (an extension of DeepWalk \cite{Perozzi:Al-Rfou:Skiena:2014}).
The basic approach is to draw a large collection of simple random walks,
view each of these walks as a ``sentence'' where each vertex is a ``word'',
and learn vertex embeddings by applying a standard word embedding method \cite{Mikolov:Chen:Corrado:Dean:2013,Mikolov:Sutskever:Chen:Corrado:2013}. 
To use mini-batch SGD, the objective function is restricted to a uniform sum over all walks.
Unbiased stochastic gradients to be computed by uniformly sampling walks.

Relational ERM models include graph representation models of this kind. 
For example, Node2Vec \cite{Grover:Leskovec:2016} is equivalent to a relational ERM model that 
(i) predicts graph structure using a predictor parameterized only by embedding vectors, 
(ii) uses a cross-entropy loss on graph structure,
and (iii) takes $\samp$ as a random-walk of fixed length (augmented with randomly sampled negative examples).

A number of other relational learning methods also 
enable SGD by restricting the objective function to a uniform sum over fixed-size subgraphs \citep[e.g.,][]{Grover:Leskovec:2016,Chamberlain:Clough:Deisenroth:2017,Yang:Cohen:Salakhudinov:2016,Hamilton:Ying:Leskovec:2017:inductive}.
Any such model is equivalent to a relational ERM model that takes $\samp$ as the uniform distribution over fixed-size subgraphs.
But, in general, relational ERM does not require restricting to sampling schemes of this kind.
Note that ``negative-sampling'' algorithms---which are critical in practice---do not uniformly sample fixed size subgraphs. 

The next examples illustrate relational ERM for problems that 
are difficult 
with existing approaches to graph representation learning.

\subsubsection*{Semi-supervised node classification}
Consider a network $G_n$ where each node $i$ is labeled by binary features---for example, 
hyperlinked documents labeled by subjects, or interacting proteins labeled by function. 
The task is to predict the labels of a subset of these nodes using the graph structure and the labels of the remaining nodes. 

The model has the following form:
Each vertex $i$ is assigned a $k$-dimensional embedding vector $\embedding_{i} \in \Reals^{k}$.
Labels are predicted using a parameterized
function ${f(\argdot;\globparam)\ :\ \Reals^{k}\to [0,1]^L}$ that maps
the node embeddings to the probability of each label.
The presence or absence of edge $i,j$ is predicted based on $\embedding_{i}^T\embedding_{j}$.
This enables learning embeddings for unlabeled vertices.
Let $\sigma$ denote the sigmoid function; let label $l_{ij}\in \{0,1\}$ denote whether vertex $i$ has label $j$;
and let $q \in [0,1]$. The loss on subgraphs $G_k \subset G_n$ is:
\begin{align} 
  &L(G_{k}; \embedding,\globparam,l) = \label{eq:model:nodes} \\ 
  &q \Bigl(\sum_{i\in\vertices(G_{k})}\sum_{j=1}^L l_{ij}\log f(\embedding_{i};\globparam)_j 
  + (1 \! - \! l_{ij})\log(1 \! - \! f(\embedding_{i};\globparam)_j)\Bigr) \nonumber \\ 
 &+ \! (1 \! - \! q)\Bigl(- \!\!\! \sum_{i,j\in\edges(G_{k})} \!\!\! \log\sigma(\embedding_{j}^{T}\embedding_{i})
 - \!\!\! \sum_{i,j\in\bar{\edges}(G_{k})} \!\!\! \log(1 \! - \! \sigma(\embedding_{j}^{T}\embedding_{i}))\Bigr). \nonumber
\end{align}
Here, $\vertices$, $\edges$, and $\bar{\edges}$ denote the vertices,
edges, and non-edges of the graph respectively. 
The loss on edge terms is cross-entropy, a standard choice in embedding models \cite{Hamilton:Ying:Leskovec:2017:review}. 
Intuitively, the predictor uses the embeddings to predict both the vertex labels and the subgraph structure.

The model is completed by choosing a sampling scheme $\samp$.
Relational ERM then fits the parameters as
\[
(\hat{\embedding}_n,\hat{\globparam}_n)
\;=\;
\argmin_{\embedding,\globparam}\EE[L \big(\embedding,\globparam ;\ \samp(G_n,k),l \big)\given G_n].
\]
We can vary the choice of $\samp$ independent of optimization concerns; 
in \cref{sec:Experiments} we observe that this leads to improved predictive performance. 

Older embedding approaches use a two-stage procedure: node embeddings are first pre-trained using the graph structure, and 
then used as inputs to a logistic regression that predicts the labels \citep[e.g.,][]{Perozzi:Al-Rfou:Skiena:2014,Grover:Leskovec:2016}. 
\citet{Yang:Cohen:Salakhudinov:2016} adapt a random-walk based method to allow simultaneous training;
their approach requires extensive development, including a custom (two-stage) variant of SGD.
Relational ERM allows simultaneous learning with no need for model specific derivation.

\subsubsection*{Wikipedia category embeddings}
\begin{figure*}[t]
  \begin{center}
    \includegraphics[width=0.65\textwidth]{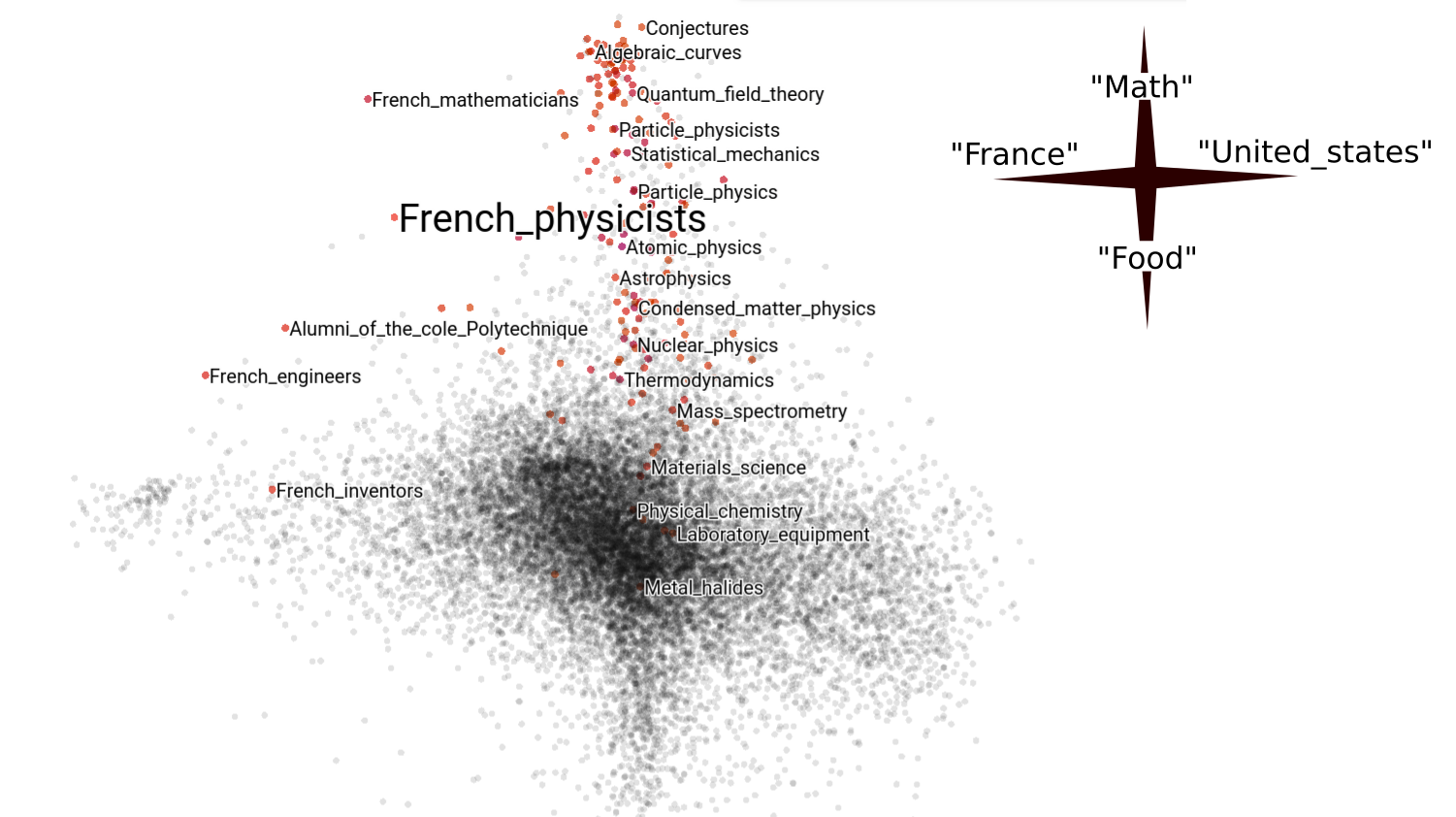};
  \end{center}
    \caption{Trained Wikipedia category embeddings. Category embeddings are projected into a 2-dimensional space,
      with a projection chosen to maximally separate ``\texttt{France}'' and ``\texttt{United\_states}'' horizontally, and ``\texttt{Math}'' and ``\texttt{Food}'' vertically.
      Highlighted categories are nearest neighbors of ``\texttt{French\_physicists}.'' 
    \label{fig:wiki_embeddings}}
\end{figure*}
We consider Wikipedia articles joined by hyperlinks. 
Each article is tagged as a member of one or more categories---for example,
``\texttt{Muscles\_of\_the\_head\_and\_neck}'', ``\texttt{Japanese\_rock\_music\_groups}'', 
or ``\texttt{People\_from\_Worcester}.'' 
The task is to learn embeddings that encode semantic relationships between the categories.

Let $G_n$ denote the hyperlink graph and let $\mathcal{C}(i)$ denote the categories of article $i$.
Each category $c \in C$ is assigned an embedding $\gamma_c$,
and the embedding of each article (vertex) is taken to be the sum of the embeddings of its categories, $\embedding_i \defeq \sum_{c\in\mathcal{C}(i)}\gamma_c$.
The loss is
\begin{align}
&L(G_{k},C; \embedding) =   \label{eq:wikipedia}\\ 
& \qquad -\!\!\!\sum_{i,j\in\edges(G_{k})}\log\sigma(\embedding_{j}^{T}\embedding_{i})
\;-\!\!\!\sum_{i,j\in\bar{\edges}(G_{k})}\log(1-\sigma(\embedding_{j}^{T}\embedding_{i})) \;, \nonumber
\end{align}
where $\edges$ and $\bar{\edges}$ denote, respectively, the presence and absence of hyperlinks between articles.
Intuitively, the predictor uses the category embeddings to predict the hyperlink structure of subgraphs.
Relational ERM chooses the embeddings as
\[
\hat{\gamma}_n=\argmin_{\gamma}\EE[L \big(\embedding(\gamma);\ \samp(G_n,k),C \big) \given G_n]\;.
\]
We write $\embedding(\gamma)$ to emphasize that the article embeddings are a function of the category embeddings.
Category embeddings obtained with this model are illustrated in \cref{fig:wiki_embeddings}; see 
\cref{sec:Experiments} for details on the experiment.

The point of this example is: relational ERM makes it easy to implement 
this non-standard relational learning model and fit it with mini-batch SGD.
The use of mini-batch SGD is important because the data graph is large.

\subsubsection*{Statistical relational learning}
Statistical relational learning takes the graph to encode the dependency structure between the units \cite[][e.g.]{Neville:Jensen:2007,Getoor:Taskar:2007}.
The idea is to infer a joint probability distribution over the entire dataset, respecting the dependency structure.
The distribution can then be used to make graph-aware predictions.
There is also work on adapting SGD to this setting \cite{Yang:Ribeiro:Neville:2017}.
Despite the similar goals, Relational ERM does not attempt to infer a distribution; the precise relationship with statistical relational learning is not clear.

\section{Subsampling algorithms}
\label{sec:samp-algorithms}

In classical ERM, sampling uniformly (with or without replacement) is typically the only choice. 
In contrast, there are many ways to sample from a graph. 
Each such sampling algorithm $\samp$ leads
to a different notion of empirical risk in \eqref{relational:risk}.

As described above, random walks underlie graph representation methods built in analogy with language models. 
A simple random walk of length $k$ on a graph $\overline{G}_n$ selects
vertices ${v_1,\ldots,v_k}$ by starting at a given vertex $v_1$, and
drawing each vertex ${v_{i+1}}$ uniformly from the neighbors of $v_i$.
Typically, random-walk based methods augment the sample by hallucinating additional edges using a strategy borrowed from the Skipgram model \cite{Mikolov:Chen:Corrado:Dean:2013}:
\begin{algorithm}[Random walk: Skipgram \cite{Perozzi:Al-Rfou:Skiena:2014}]
  \label{alg:rwsg}
\leavevmode
  \begin{enumerate}[label=(\roman*),topsep=0pt]
    \item Sample a random walk ${v_1,\ldots,v_k}$ starting at a uniformly selected vertex of $\overline{G}_n$.
    \item Report $\overline{G}_k = \{(v_i,v_j) : d(v_i,v_j) < W \}$. The \emph{window} $W$ is a sampler parameter, and $d(v_i,v_j)$ is the number of steps between $v_i$ and $v_j$.
  \end{enumerate}
\end{algorithm}
Since relational ERM is indifferent to the connection with language models, a natural alternative augmentation strategy is:
\begin{algorithm}[Random walk: Induced]
  \label{alg:rwinduced}
  \leavevmode
  \begin{enumerate}[label=(\roman*),topsep=0pt]
  \item Sample a random walk ${v_1,\ldots,v_k}$ starting at a uniformly selected vertex of $\overline{G}_n$.
   \item Report $\overline{G}_k$ as the edge list of the vertex induced subgraph of the walk.
    \end{enumerate}
\end{algorithm}

A simple choice is to sample $k$ vertices uniformly at random and report $\overline{G}_k$ as the induced subgraph.
Such an algorithm will not work well in practice since it is not suitable for sparse graphs.
We are typically interested in the case $k \ll n$.
If $\overline{G}_n$ is sparse then such a sample typically includes few or no edges,
and thus carries little information about $\overline{G}_n$. 
The next algorithm modifies uniform vertex sampling to fix this pathology.
The idea is to over-sample vertices and 
retain only those vertices that participate in at least one edge in the induced subgraph.
\begin{algorithm}[$p$-sampling \cite{Veitch:Roy:2016}] 
  \label{alg:p}
  \leavevmode
  \begin{enumerate}[label=(\roman*),topsep=0pt]
    \item Select each vertex in $\overline{G}_n$ independently,
      with a fixed probability ${p\in[0,1]}$.
    \item Extract the induced subgraph ${\overline{G}_k}$ of 
      $\overline{G}_n$ on the selected vertices.
    \item Delete all isolated vertices
      from $\overline{G}_k$, and report the resulting graph.
   \end{enumerate}
\end{algorithm}

Another natural sampling scheme is:
\begin{algorithm}[Uniform edge sampling]
  \label{alg:u:edges}
   \leavevmode
  \begin{enumerate}[label=(\roman*),topsep=0pt]
  \item Select $k$ edges in $\overline{G}_n$ uniformly and independently from the edge set.
   \item Report the graph $\overline{G}_k$ consisting of these edges, and all vertices incident to these edges.
  \end{enumerate}
\end{algorithm}
Many other sampling schemes are possible; see \citet{Leskovec:Faloutsos:2006} for a discussion of possible options in a related context.

\subsection{Negative sampling}
For a pair of vertices in an input graph $\overline{G}_n$, a sampling algorithm can report three types of edge information:
The edge may be observed as present, observed as absent (a \emph{non-edge}), or may not be observed.
The algorithms above do not treat edge and non-edge information equally:
\cref{alg:u:edges,alg:rwinduced,alg:rwsg} cannot report non-edges, 
and the deletion step in \cref{alg:p} biases it towards edges over non-edges.
However, the locations of non-edges can carry significant information.

Negative sampling schemes are ``add-on'' algorithms that are applied to the output of a graph sampling algorithm
and augment it by non-edge information. Let $\overline{G}_k$ denote a sample generated by one of the algorithms above
from an input graph $\overline{G}_n$. 
\begin{nalgorithm}[Negative sampling: Induced]
  \label{alg:induced}
  \leavevmode
  \begin{enumerate}[label=(\roman*),topsep=0pt]
\item Report the subgraph induced by $\overline{G}_k$, in the input graph $\overline{G}_n$
      from which $\overline{G}_k$ was drawn.
  \end{enumerate}
\end{nalgorithm}
Another method, originating in language modeling \cite{Mikolov:Sutskever:Chen:Corrado:2013,Goldberg:Levy:2014},
is based on the unigram distribution:
Define a probability distribution
on the vertex set of $\overline{G}_k$ by ${P_{n}(v):=\text{Prob}\lbrace v\in\overline{H}_k\rbrace}$, the probability
that $v$ would occur in a separate, independent sample $\overline{H}_k$ generated from $\overline{G}_n$ by the same algorithm
as $\overline{G}_k$.
For $\tau>0$, we define
a distribution ${P_n^{\tau}(v):=(P_{n}(v))^{\tau}/Z(\tau)}$, where $Z(\tau)$ is the appropriate normalization.
\begin{nalgorithm}[Negative sampling: Unigram]
  \label{alg:Unigram}
  For each vertex $v$ in $\overline{G}_n$:
  \leavevmode
  \begin{enumerate}[label=(\roman*),topsep=0pt]
  \item Select $k$ vertices ${v_1,\ldots,v_k \distiid P_n^{\tau}}$.
    \item If $(v,v_j)$ is a non-edge in $\overline{G}_n$, add it to $\overline{G}_n$.
  \end{enumerate}
\end{nalgorithm}
The canonical choice in the embeddings literature is ${\tau=\frac{3}{4}}$ \cite{Mikolov:Sutskever:Chen:Corrado:2013}.

\section{Theory}
\label{sec:theory} 

\newcommand{\rERM}{\hat{R}}
\newcommand{\smptheta}{\hat{\embedding}^{\samp}}
\newcommand{\smplambda}{\hat{\globparam}^{\samp}}
\newcommand{\smplambdalim}{\globparam^{\samp}}

\newcommand{\pERM}{\hat{R}^{\mathrm{ps}}}
\newcommand{\rwERM}{\hat{R}^{\mathrm{rw}}}
\newcommand{\ptheta}{\hat{\embedding}^{\mathrm{ps}}}
\newcommand{\rwtheta}{\hat{\embedding}^{\mathrm{rw}}}
\newcommand{\plambda}{\hat{\globparam}^{\mathrm{ps}}}
\newcommand{\rwlambda}{\hat{\globparam}^{\mathrm{rw}}}
\newcommand{\plambdalim}{\globparam^{\mathrm{ps}}_{*}}
\newcommand{\rwlambdalim}{\globparam^{\mathrm{rw}}_{*}}

\newcommand{\dataset}{\overline{G}}
\newcommand{\gpopulation}{\mathcal{G}}

\newcommand{\restrict}{|}

We now turn to formalizing and establishing theoretical properties of relational ERM.
Particularly, (i) relational ERM satisfies basic theoretical desiderata, and (ii) $\samp$ should be viewed as a model component.
We first give the results, and then discuss their interpretation and significance.

When the data is unstructured (i.e., no link structure), theoretical analysis of ERM relies on the assumption that the data is generated i.i.d.
The \iid assumption is ill-defined for relational data. 
Any analysis %
requires some analogous assumption for how the data $\overline{G}_n$ is generated. 
Following recent work %
emphasizing the role of sampling theory in modeling graph data \cite{Orbanz:2017,Veitch:Roy:2016,Borgs:Chayes:Cohn:Veitch:2017,Crane:Dempsey:2016:snm},
we model $\overline{G}_n$ as a random sample drawn from some large population network.
Specifically, we consider a population graph $\mathcal{G}$ with $|\mathcal{G}|$ edges,
and assume that the observed sample $\overline{G}_n$ of size $n$ is generated by $p$-sampling
from $\mathcal{G}$, with ${p=n/\sqrt{|\mathcal{G}|}}$. 
We assume the population graph is ``very large,'' in the sense that ${|\mathcal{G}|\rightarrow\infty}$.
The distribution of $\overline{G}_n$ in the ``infinite population'' case is well-defined \cite{Borgs:Chayes:Cohn:Veitch:2017}.

The analogy with \iid data generation is two-fold: 
Foundationally, the \iid assumption is equivalent to assuming the data 
is collected by uniform sampling from some population \cite{Politis:Romano:Wolf:1999},
and $p$-sampling is a direct analogue \cite{Veitch:Roy:2016,Borgs:Chayes:Cohn:Veitch:2017,Orbanz:2017}. 
Pragmatically, both assumptions strike a balance between 
being flexible enough to capture real-world data \cite{Caron:Fox:2017,Veitch:Roy:2015}
and simple enough to allow precise theoretical statements.

We establish results for several choices of $\samp(G_n,k)$.
Edges may be selected by either $p$-sampling with ${p=k/\sqrt{n}}$---note the size of $\samp(G_n,k)$ is free of $n$---or 
by using a simple random walk of length $k$ (\cref{alg:rwsg} or \cref{alg:rwinduced}).
Negative examples may be chosen by \cref{alg:induced} or \cref{alg:Unigram}.
The main result guarantees that the limiting risk of the parameter we learn depends
only on the population and the model, and not on idiosyncrasies of the training data.
\begin{theorem}
\label{thm:emp_risk_conv}
Suppose that $G_n$ is collected by $p$-sampling as described above,
that $k \in \Nats$ is fixed, and
that $\samp$ is fixed to a sampling algorithm based on either $p$-sampling or random walk sampling as described above.
Suppose further that the loss is bounded and parameter setting $\bar{\genparam} = (\bar{\globparam},\bar{\embedding})$ satisfies mild technical conditions given in the appendix.
Then there is some constant $c_{\bar{\genparam}}(\samp, k) \in \NNReals$ such that
\begin{equation}
\label{eq:thm:1}
\rERM_{k}(\bar{\genparam}; \dataset_n) \to c_{\bar{\genparam}}(\samp, k)
\end{equation}
both in probability and in $L_{1}$ as $n\to\infty$.
Moreover, there is some constant $c_{*}(\samp, k) \in\NNReals$ such that 
\begin{equation}
\label{eq:thm:2}
\min_{\genparam}\rERM_{k}(\genparam; \dataset_n) \to c_{*}(\samp, k)
\end{equation}
both in probability and in $L_{1}$, as $n\to\infty$.

The limits depend on the choice of $\samp$ (and $k$), and usually do not agree between different sampling schemes.
\end{theorem}
The result is proved for $\samp$ based on $p$-sampling in \cref{sec:psamp} and for random-walk based sampling in \cref{sec:randomwalk}.

Classical ERM guarantees usually apply even to the parameter itself, not just its risk.
In the relational setting, the possibly complicated interplay of the learned embeddings makes such
results more difficult. The next two results build on \cref{thm:emp_risk_conv} 
to establish (partial) guarantees for the parameter itself.

We establish a convergence result for the global parameters output by a two-stage
procedure where the embedding vectors are learned first. 
Such a result is applicable, for example, when predicting vertex attributes from embedding vectors
that are pre-trained to explain graph structure.
The proof is given in \cref{sec:glob_param_conv}.
\begin{theorem}\label{thm:glob_conv}
Suppose the conditions of \cref{thm:emp_risk_conv}, and 
also that the loss function verifies a certain strong convexity property in $\globparam$, given
explicitly in the appendix.
Let $\tilde{\globparam}_n = \argmin_{\globparam}\min_{\embedding}\rERM_{k}(\globparam,\embedding;\dataset_n)$.
Then $\tilde{\globparam}_n \to \tilde{\globparam}_{*}(\samp,k)$ in probability 
for some constant $\tilde{\globparam}_{*}(\samp,k)$.
\end{theorem}

We next establish a stability result showing that collecting additional data does not dramatically change learned embeddings.  
The proof is given in \cref{sec:stability-of-embeddings}.
\begin{theorem}
Suppose the conditions of \cref{thm:emp_risk_conv},
and also that the loss function is twice differentiable and the Hessian of the empirical risk is bounded. 
Let $\hat{\embedding}_{n+1}\restrict_n$ denote the restriction of the embeddings $\hat{\embedding}_{n+1}$
to the vertices present in $G_n$.
Then
$
\hat{\embedding}_{n} - \hat{\embedding}_{n+1}\restrict_n \to 0
$
in probability, as $n\to\infty$.
\end{theorem}

The examples of \cref{sec:example_models} do not satisfy the conditions of the theorem because the cross-entropy loss is unbounded.
However, the models can be trivially modified to bound the output probabilities away from 0 and 1.
In this case, the loss is bounded. Further, for the logistic regression model used in the 
experiments the convexity and Hessian conditions also hold, by direct computation.

\subsubsection*{Interpretation and Significance}
The properties we establish are minimal desiderata that one might demand of any sensible learning procedure.
Nevertheless, such results have not been previously established for relational learning methods.
The obstruction is the need for a suitable analogue of the \iid assumption.
The demonstration that population sampling can fill this role is itself a main contribution of the paper.
Indeed, the results we establish are weaker than the analogous guarantees for classical ERM,
and main significance is perhaps the demonstration that such results can be established at all.
This is important both as a foundational step towards a full theoretical analysis of relational learning,
and because it strengthens the analogy with classical ERM.

A strength of our arguments is that they are largely agnostic to the particular choice of model,
mitigating the need for model-specific analysis and justification.
For example, our results include random-walk based graph representation methods as a special case, 
providing some post-hoc support for the use of such methods.

The limits in \cref{thm:emp_risk_conv,thm:glob_conv} depend on the choice of $\samp$.
Accordingly, the limiting risk and learned parameters depend on $\samp$ in the same sense they depend on the choice of predictor
class and the loss function; i.e., $\samp$ is a model component.
This underscores the need to consider the choice in model design, 
either through heuristics---e.g., random-walk sampling upweights the importance of high degree vertices relative to $p$-sampling---or
by trying several choices experimentally.

\section{Experiments}
\label{sec:Experiments}

The practical advantages of using relational ERM  to define new, task-specific, models are: 
(i) Mini-batch SGD can be used in a plug-and-play fashion to solve the optimization problem.
This allows inference to scale to large data. And, (ii) by varying $\samp$ we may improve model quality.
We have used relational ERM to define novel models in \cref{sec:example_models}. 
The models are determined by \eqref{eq:model:nodes} and \eqref{eq:wikipedia}
up to the choice of $\samp$.
We now study these example models empirically.\footnote{Code at \href{https://github.com/wooden-spoon/relational-ERM}{github.com/wooden-spoon/relational-ERM}}
The main observations are: (i) SGD succeeds in quickly fitting the models in all cases.
And, (ii) the choice of $\samp$ has a dramatic effect in practice. 
Additionally, we observe that the best model for the semi-supervised node classification task uses $p$-sampling.
$p$-sampling has not previously been used in the embedding literature, 
and is very different from the random-walk based schemes that are commonly used.

\subsection*{Node classification problems}
\newcommand{\maxf}[1]{{\cellcolor[gray]{0.8}} #1}

We begin with the semi-supervised node classification task 
described in \cref{sec:example_models},
using the model \cref{eq:model:nodes} with different choices of $\samp$.
\begin{wraptable}{r}{5.cm}
 \begin{tabular}{@{} rrr @{}}
   & {Blogs} & {Protein}\tabularnewline
   \midrule
   Vertices & 10,312 & 3,890 \\
   Edges & 333,983 & 76,584  \\
   Label Dim. & 39 & 50 \\
 \end{tabular}
\end{wraptable} 
We study the blog catalog and protein-protein interaction data reported in
\cite{Grover:Leskovec:2016}, summarized by the table to the right.
We pre-process the data to remove self-edges, and 
restrict each network to the largest connected component.
Each vertex in the graph is labeled, 
and $50\%$ of the labels are censored at training time. 
The task is to predict these labels at test time.

\paragraph{Two-stage training.} 
\begin{table}[h!]
\caption{Average Macro-F1 for Two-Stage Training.}\label{tb:two-stage-semisup}
\begin{center}
\begin{tabular}{lc*{2}{S}}
Choice of $\samp$ & {Alg.\ \#} & {Blogs} & {Protein}\tabularnewline
\midrule
rw/skipgram+ns & \ref{alg:rwsg}+\ref{alg:Unigram} & 0.18 & {\maxf{0.16}} \\
rw/induced+ind & \ref{alg:rwinduced}+\ref{alg:induced} & 0.08 & 0.08 \\
rw/induced+ns & \ref{alg:rwinduced}+\ref{alg:Unigram} &  0.18 & {\maxf{0.16}} \\
$p$-samp+ind. & \ref{alg:p}+\ref{alg:induced} & 0.17  & 0.14 \\
$p$-samp+ns & \ref{alg:p}+\ref{alg:Unigram} & {\maxf{0.22}} &  {\maxf{0.16}} \\
unif. edge+ns & \ref{alg:u:edges}+\ref{alg:Unigram} & 0.21 & 0.15 \\
\label{tab:Sampling_comparision}
\end{tabular}
\end{center}
\end{table}

\begin{table*}[t]
\caption{Average Macro-F1 for Simultaneous Training. Columns are labeled by the sampling scheme used to draw test vertices.}\label{tb:simultaneous-semisup}
\begin{center}
   \begin{tabular}{l @{\hskip 3em} *{3}{S} @{\hskip 3em} *{3}{S}}
     & \multicolumn{3}{c}{Blog catalog}
     & \multicolumn{3}{c}{Protein-Protein}\\
     $\samp$ & {Unif.} & {$p$-samp} & {rw} & {Unif.} & {$p$-samp} & {rw} \\ 
     \midrule
     rw/skipgram+ns (Alg.\ \ref{alg:rwsg}+\ref{alg:Unigram}) & 0.20  & 0.26 & 0.27 & 0.25 & .32 & 0.34  \\ 
     $p$-samp+ns (Alg.\ \ref{alg:p}+\ref{alg:Unigram}) & {\maxf{0.30}} & 0.34  & 0.35 & {\maxf{0.30}} & 0.37  & 0.39  \\ 
     Node2Vec (reported) & 0.26 & {-} & {-} & 0.18 & {-} & {-}
   \end{tabular}
\end{center}
\end{table*}

We first train the model \eqref{eq:model:nodes} using no label information
to learn the embeddings (that is, with ${q=0}$).
We then fit a logistic regression to predict vertex features from the trained embeddings. 
This two stage approach is a standard testing procedure in the graph embedding literature, e.g.\ \cite{Perozzi:Al-Rfou:Skiena:2014,Grover:Leskovec:2016}. 
We use the same scoring procedure as Node2Vec \cite{Grover:Leskovec:2016} (average macro F1 scores), 
and, where applicable, the same hyperparameters. 

\cref{tb:two-stage-semisup} shows the effect of varying the sampling scheme used to train the embeddings.
As expected, we observe that the choice of sampling scheme affects the embeddings produced via the learning procedure,
and thus also the outcome of the experiment.
We further observe that sampling non-edges by unigram negative sampling gives better predictive
performance relative to selecting non-edges from the vertex induced subgraph.

\paragraph{Simultaneous training.} 
Next, we fit the model of \cref{sec:example_models} with $q=0.001$---training the embeddings and global variables simultaneously.
Recall that simultaneous training is enabled by the use of relational ERM. 
We choose label predictor $\pi_\globparam$ as logistic regression, and 
adapt the label prediction loss to measure the loss only on vertices in the positive sample. 

There is not a unique procedure for creating a test set for relational data.
We report test scores for test-sets drawn according to several different sampling schemes.
Results are summarized by \cref{tb:simultaneous-semisup}.
We observe:
\begin{itemize}
\item Simultaneous training improves performance.
\item $p$-sampling outperforms the standard rw/skipgram procedure. 
\item This persists irrespective of how the test set is selected (i.e., it is not an artifact of the data splitting procedure).
\end{itemize}
Note that the average computed with uniform vertex sampling is the standard scoring procedure used in the previous table.
The last observation is somewhat surprising: we might have expected a mismatch between the training and testing objectives to degrade performance.
One possible explanation is that the random-walk based sampler excessively downweights low-connectivity vertices, and thus fails to fully exploit their label information.

\subsection*{Wikipedia Category Embeddings}
We consider the task of discovering semantic relations between Wikipedia categories, as described in \cref{sec:example_models}. 
This task is not standard; wholly new model is required. 

We define a relational ERM model by choosing 
category embedding dimension $k=128$,
the loss function $L$ in \eqref{eq:wikipedia}, and $\samp$ as \ref{alg:rwsg}+\ref{alg:Unigram}, the skipgram random walk sampler with unigram negative sampling.
The data $\overline{G}_n$ is the Wikipedia hyperlink network from \cite{Klymko:Gleich:Kolda:2014},
consisting of Wikipedia articles from 2011-09-01 restricted to articles in categories containing at least 100 articles.

The challenge for this task is that the dataset is relatively large---about 1.8M nodes and 28M edges---and
the model is unusual---embeddings are assigned to vertex attributes instead of the vertices themselves.
SGD converges in about 90 minutes on a desktop computer equipped with a Nvidia Titan Xp GPU.
\cref{fig:wiki_embeddings} on page \pageref{fig:wiki_embeddings} visualizes example trained embeddings,
which clearly succeed in capturing latent semantic structure.

\section{Conclusion}

Relational ERM is a generalization of ERM from \iid data to relational data.
The key ideas are introducing $\samp$ as a component of model design, 
which defines an analogue of the empirical distribution,
and using the assumption that the data is sampled from a population network as an analogue of the \iid assumption.
Relational ERM models can be fit automatically using SGD.
Accordingly, relational ERM provides an easy method to specify and fit relational data models.

The results presented here suggest a number of directions for future inquiry.
Foremost: what is the relational analogue of statistical learning theory?
The theory derived in \cref{sec:theory} establishes initial results.
A more complete treatment may provide statistical guidelines for model development. 
Our results hinge critically on the assumption that the data is collected by $p$-sampling;
it is natural to ask how other data-generating mechanisms can be accommodated.
Similarly, it is natural to ask for guidelines for the choice of $\samp$.

\subsubsection*{Acknowledgments}
VV and PO were supported in part by grant FA9550-15-1-0074 of AFOSR.
DB is supported by ONR N00014-15-1-2209, ONR 133691-5102004, NIH 5100481-5500001084, NSF CCF-1740833, the Alfred P. Sloan Foundation, the John Simon Guggenheim Foundation, Facebook, Amazon, and IBM.
The Titan Xp used for this research was donated by the NVIDIA Corporation.

\printbibliography

\newpage
\appendix
\onecolumn

\title{Proofs of Theoretical Results for Empirical Risk Minimization and Stochastic Gradient Descent for Relational Data}
\date{\vspace{-5ex}}
\maketitle

\newcommand\independent{\protect\mathpalette{\protect\independenT}{\perp}}
\def\independenT#1#2{\mathrel{\rlap{$#1#2$}\mkern2mu{#1#2}}}

\begin{refsection}

\section{Overview of Proofs}
\newcommand{\KEG}{\Gamma}
\newcommand{\gloc}{x}
\newcommand{\glocc}{y}
\newcommand{\glab}{\nu}
\newcommand{\st}{\ :\ }

The appendix is devoted to proving the theoretical results of the paper.  
These results are obtained subject to the assumption that the data is collected by $p$-sampling. 
This assumption is natural in the sense that it provides a reasonable middle ground between a realistic data collection 
assumption---$p$-sampling can result in complex models capturing many 
important graph phenomena \cite{Caron:Fox:2017,Veitch:Roy:2015,Borgs:Chayes:Cohn:Holden:2016}---and 
mathematical tractability---we are able to establish precise guarantees. 

The appendix is organized as follows. 
We begin by recalling the connection between $p$-sampling and \emph{graphex processes} in \cref{graphex_processes};
this affords a useful explicit representation of the data generating process. 
In \cref{exchangeable_pairs}, we recall the method of exchangeable pairs, a technical tool required for our convergence proofs.
Next, in \cref{notation}, we collect the necessary notation and definitions. 
Empirical risk convergence results for $p$-sampling are then proved in \cref{sec:psamp} and results for the random-walk in \cref{sec:randomwalk}.
Convergence results for the global parameters are established in \cref{sec:glob_param_conv}.
Finally, in \cref{sec:stability-of-embeddings}, we show that learned embeddings are stable in sense that they are not changed much by collecting a small amount of additional data.

\section{Preliminaries}
\subsection{Graphex processes}
\label{graphex_processes}
Recall the setup for the theoretical results: we consider a very large population network
$P_{t}$ with $t$ edges, and we study the graph-valued stochastic
process $(G_{n}^{t})_{n\in[0,\sqrt{t})}$ given by taking each $G_{n}^{t}$
to be an $n/\sqrt{t}$-sample from $P_{t}$ and requiring these samples
to cohere in the obvious way. We idealize the population size as infinite
by taking the limit $t\to\infty$. The limiting stochastic process $(G_{n})_{n\in\NNReals}$
is well defined, and is called a \emph{graphex process} \cite{Borgs:Chayes:Cohn:Veitch:2017}. 

Graphex processes have a convenient explicit representation
in terms of (generalized) \emph{graphons} \cite{Veitch:Roy:2015,Borgs:Chayes:Cohn:Holden:2016,Caron:Fox:2017}.
\begin{defn}
A \defnphrase{graphon} is an integrable function $W:\NNReals^2 \to [0,1]$.
\end{defn} 
\begin{remark}
This notion of graphon is somewhat more restricted than graphons (or graphexes) considered in full generality, but it suffices 
for our purposes and avoids some technical details. 
\end{remark}

We now describe the generative model for a graphex process with graphon $W$.
Informally, a graph is generated by 
(i) sampling a collection of vertices $\{\glab_i\}$ each with latent features $\gloc_i$,
and (ii) randomly connecting each pair of vertices with probability dependent on the latent features.    
Let 
\[
\PP=\{\eta_i\}_{i\in \Nats}=\{(\glab(\eta_i),\gloc(\eta_{i}))\}_{i\in \Nats}
\]
 be a Poisson (point) process on $\NNReals \times \NNReals$
with intensity $\Lebesgue \otimes \Lebesgue$, where $\Lebesgue$ is the Lebesgue measure.
Each atom of the point process is a candidate vertex of the sampled graph;
the $\{\glab_i\}$ are interpreted as (real-valued) labels of the vertices, and the $\{\gloc_i\}$ as latent features that explain the graph structure.
Each pair of points $(\eta_i, \eta_j)$ with $i \le j$ is then connected independently according to
\[
1[(\eta_i, \eta_j)\text{ connected}] \distind \bernDist(W(\gloc_i, \gloc_j)).
\]
This procedure generates an infinite graph.  
To produce a finite sample of size $n$,
we restrict to the collection of edges $\KEG_n = \{(\eta_i, \eta_j) \st \eta_i, \eta_j \le n \}$.
That is, we report the subgraph induced by restricting to vertices with label less than $n$,
and removing all vertices that do not connect to any edges in the subgraph.
This last step is critical; in general there are an infinite number of points of the Poisson process
such that $\eta_i < n$, but only a finite number of them will connect to any edge in the induced subgraph.

Modeling $G_n$ as collected by $p$-sampling is essentially equivalent to positing that
$G_n$ is the graph structure of $\KEG_n$ generated by some graphon $W$.
Strictly speaking, the $p$-sampling model induces a slightly more general generative model that 
allows for both isolated edges that never interact with the main graph structure, 
and for infinite star structures; see \cite{Borgs:Chayes:Cohn:Veitch:2017}.
Throughout the appendix, we ignore this complication and assume that the dataset graph is generated by some graphon.  
It is straightforward but notationally cumbersome to extend our results to $p$-sampling in full generality.

\subsection{Technical Background: Exchangeable Pairs}
\label{exchangeable_pairs}
We will need to bound the deviation of the (normalized) degree of a vertex from its expectation.
To that end, we briefly recall the method of exchangeable pairs; see \cite{Chaterjee:2005} for details.

\begin{defn}
A pair of real random variables $(X,X')$ is said to be exchangeable if $$(X,X')\overset{d}{=}(X',X).$$
\end{defn}
Let $f:\mathbb{R}\rightarrow \mathbb{R}$ and $F:\mathbb{R}^2\rightarrow \mathbb{R}$ be measurable function such that: $$\mathbb{E}(F(X,X')|X)\overset{a.s}{=}f(X),~\rm{and}~ F(X,X')=-F(X',X).$$

Let 
\[
v(X)\triangleq \frac{1}{2}\mathbb{E}\Big(\big(f(X)-f(X')\big)F(X,X')\Big|X\Big),
\] 
and suppose that $|v(X)|\overset{a.s}{\le} C$ for some $C\in \mathbb{R}$.
Then 
\[
\forall x>0,~P(|f(X)-\mathbb{E}(f(X))|\ge x)\le 2 e^{-\frac{x^2}{2C}}.
\]
Further,  for all $p>1$  and $x>0$ it holds that: 
\[
P(|f(X)-\mathbb{E}(f(X))|>x)\le \frac{(2p-1)^p\|v(X)|\|^p_p}{x^p}.
\]

\subsection{Notation}
\label{notation}
\newcommand{\paths}{\mathcal{P}}
\newcommand{\ppoint}{\eta} %

For convenient reference, we include a glossary of important notation.

First, notation to refer to important graph properties:
\begin{itemize}
\item $\PP = \{\ppoint_i = (\glab(\ppoint_i),\gloc(\ppoint_i))\}$ is the latent Poisson process that defines the graphex process in \cref{graphex_processes}.
The labels are $\glab$ and the latent variables are $\gloc$. 
\item $\PP_n\triangleq \PP \cap [0,n] \times \mathbb{R}^+$ is the restriction of the Poisson process to atoms with labels in $[0,n]$.
\item  To build the graph from the point of process $\PP_n$ we need to introduce a process of independent uniform variables. 
Let
\[
\mathbb{U}_{\Pi}\triangleq (U_{\ppoint_i,\ppoint_j})_{\ppoint_i,\ppoint_j\in \Pi}
\] be such that $\mathbb{U}_{\Pi}|\Pi$ is an independent process where $U_{\ppoint_1,\ppoint_2}\given \Pi \distiid \uniDist(0,1)$

\item $\KEG_n \subset \NNReals^2$ is the (random) edge set of the graphex process at size $n$.
\item $V(\KEG_n) \subset \NNReals$ is the set of vertices of $\KEG_n$. 
\item $\bar{\KEG}_n = \{(\ppoint_i,\ppoint_j) \st \ppoint_i,\ppoint_j \in V(\KEG_n) \text{ and } (\ppoint_i,\ppoint_j) \notin \KEG_n\}$ is all pairs of points in $\KEG_n$ that 
are not connected by an edge.
\item The number of edges in the graph is $E_n = \abs{\KEG_n}$ %
\item The neighbors of $\eta$ in $\KEG_n$ are
\[
\mathcal{N}_n(\eta)\triangleq \{\eta' \st (\eta,\eta')\in \mathcal{P}_1(\KEG_n)\}
\]
\item For all $k$, the set of paths of length $k$ in $\KEG_n$ is 
\[
\paths_k(\KEG_n) \triangleq \{(\ppoint_i)_{i\le k+1}\in V(\KEG_n)^{k+1} \st ~(\ppoint_i,\ppoint_{i+1})\in \KEG_n\ \forall i\le k \}.
\]
\item The degree of $\glab$ in $\KEG_n$ is $d_n(\ppoint )$.
\item Asymptotically, the number of edges of a graphex process scales as $n^2$ \cite{Borgs:Chayes:Cohn:Holden:2016}. 
Let  $\mathcal{E}\in \NNReals$ be the proportionality constant 
\[
\mathcal{E}\triangleq \lim_{n\rightarrow \infty} \frac{E_n}{n^2}.
\]

\end{itemize}

Next, we introduce notation relating to model parameters. 
Treating the embedding parameters requires some care.
The collection of vertices of the graph is a random quantity, and so the embedding parameters must also be modeled as random.
For graphex processes, this means the embedding parameters depend on the latent Poisson process used in the generative model.
To phrase a theoretical result, it is necessary to assume something about the structure of the dependence.
The choice we make here is: the embedding parameters are taken to be markings of the Poisson process $\PP$.
In words, the embedding parameter of a vertex may depend on the (possibly latent) properties of that vertex, but the embeddings
are independent of everything else. 
\begin{itemize}
\item  
The collection of all possible parameters is: 
\[
\Omega_{\genparam}^{\Pi}\triangleq \{(\embedding_{\ppoint},\globparam)_{\ppoint\in \Pi} \st \embedding_{\ppoint}\in \Omega_{\genparam}\ \forall \ppoint \in \Pi \text{ and }\globparam \in \Omega_{\globparam}\}.
\]
Note that we attach a copy of the global parameter to each vertex for mathematical convenience.

\item For all $\bar{\genparam}\in \Omega_{\genparam}^{\Pi}$, let $\embedding(\bar{\genparam})$ denote the projection on $\Omega_{\embedding}^{\Pi}$ and let $\globparam(\bar{\genparam})$ denote the projection on $\Omega_{\globparam}.$
\item The following concepts and notations are needed to build a marking of the Poisson process:
Let $m(\cdot,\cdot)$ be a distributional kernel on  $\mathbb{R}_+\times \Omega_{\genparam}$. We generate the marks according to a distribution $\mathcal{Q}_{\genparam}^{\Pi}$ on $\Omega_{\genparam}^{\Pi}$, conditional on $\PP$, such that if $\bar{\genparam}|\PP\sim \mathcal{Q}_{\genparam}^{\PP}$ then:
\begin{itemize}
 \item  $(\bar{\genparam}_{\eta})_{\ppoint \in \Pi}$ is an independent process 
\item  $\bar{\genparam}_{\eta}|\Pi\sim m(\gloc(\ppoint),\cdot)$ for all $\eta \in \PP$
\end{itemize}

\item Let $\bar{\PP}_n({\genparam})\triangleq (\PP_n,\mathbb{U}|_{\PP_n},\genparam|_n)$ the augmented object 
that carries information about both the graph structure ($\PP_n,\mathbb{U}|_{\PP_n}$) and the model parameters $\genparam$.

\end{itemize}

\section{Basic asymptotics for $p$-sampling}
\label{sec:psamp}

We begin by establishing the result for $p$-sampling, with $p=k/\sqrt{n}$ and 
the non-edges chosen by taking the induced subgraph.
This is the simplest case, and is useful for the introduction of ideas and notation. 
We consider more general approaches to negative sampling in the next section, where it is treated in tandem with random walk sampling.
The same arguments can be used to extend $p$-sampling to allow for, e.g., unigram negative sampling used in our experiments.

For all ${\bar{\genparam}}\in \Omega_{\genparam}^{\PP}$,
and all $\KEG'_k \subset \KEG$, let $L(\KEG'_k, \bar{\genparam})$ 
denote the loss on $\KEG'_k$ where $\bar{\genparam}$ is restricted to the embeddings (and global parameters) associated with $\KEG'_k$.

\begin{theorem}
\label{thm:p-samp_convergence}
Let $\bar{\genparam}$ a random variable taking value in $\Omega_{\genparam}^{\Pi}$ such that $\bar{\genparam}~|\Pi\sim\mathcal{Q}_{\genparam}^{\Pi}$, for a certain kernel $m$, 
then there  is some constant $c^{\mathrm{ps}}_{m}\in\NNReals$ such that if $\|\mathcal{L}\|_{\infty}<\infty$ then 
\[
\hat{R}_k( \KEG_n,  \bar{\genparam}) \to c^{\mathrm{ps}}_{m}
\] 
both \as\ and in $L_{1}$, as $n\to\infty$.

Moreover there is some constant $c^{\mathrm{ps}}_{*}\in\NNReals$ such that 
\[
\min_{\genparam}\hat{R}_k( \KEG_n, \genparam ) \to c^{\mathrm{ps}}_{*}
\]
 both \as\ and in $L_{1}$, as $n\to\infty$. 
\end{theorem}
\begin{proof} We will first prove the first statement. Let $\bar{\genparam}|\Pi\sim \mathcal{Q}_{\genparam}^{\Pi}$, let $\KEG(\bar{\genparam})$ be the edge set of $\bar{\Pi}(\bar{\genparam})$, and 
let $\KEG^n(\bar{\genparam})$ be the partially labeled graph obtained from $\KEG(\bar{\genparam})$ by forgetting all labels in $[0,n)$ (but keeping larger labels and the embeddings $\genparam$).
Let $\mathcal{F}_n(\bar{\genparam})$ be the $\sigma$-field generated by $\KEG^n(\bar{\genparam})$. The critical observation is  
\begin{equation}\label{p_sam}
\hat{R}_k( \KEG_n,  \bar{\genparam}) = \EE[L(\KEG_k,\bar{\genparam}) \given \mathcal{F}_n(\bar{\genparam})].
\end{equation}
The reason is that choosing a graph by $k/n$-sampling is equivalent uniformly relabeling the vertices in $[0,n)$ and restricting to labels less than $k$;
averaging over this random relabeling operation is precisely the expectation on the righthand side.%

By the reverse martingale convergence theorem we get that:
\[
\hat{R}_k( \KEG_n,  \bar{\genparam}) \xrightarrow{\as,L_1}  \EE[L(\KEG_k, \bar{\genparam}) \given \mathcal{F}_{\infty}(\bar{\genparam})],
\] 
but as $\mathcal{F}_{\infty}(\bar{\genparam})$ is a trivial sigma-algebra we get the desired result.

We will now prove the second statement. 
Let $\KEG^n$ be the partially labeled graph obtained from $\KEG$ by forgetting all labels in $[0,n)$ and
let $\mathcal{F}_n$ be the $\sigma$-field generated by $\KEG^n$. Further, we denote the set of embeddings of the graph $\KEG^m$ by: 
\[
\Omega_{\genparam}^{\KEG^m}\triangleq \{(\embedding_{\mathcal{V},\globparam})_{\mathcal{V}\in \KEG^m}:\forall \mathcal{V}\in V(\KEG^m) ~\embedding_{\mathcal{V}}\in \Omega_{\embedding}, \globparam \in \Omega_{\globparam}\}.
\]

We are now ready to state the proof.
Let $m\le n$, and observe that:
\begin{align}
\EE[\min_{\genparam\in \Omega_{\genparam}^{\KEG^n}} \hat{R}_k(\KEG_n, \genparam) \given \mathcal{F}_m] &\le \min_{\genparam\in \Omega_{\genparam}^{\KEG^m}} \EE[ L(\KEG_k,\genparam) \given \mathcal{F}_m]\\
& = \min_{\genparam\in \Omega_{\genparam}^{\KEG^m}} \hat{R}_k(\KEG_n, \genparam).
\end{align}
Thus, $(\min_{\genparam\in \Omega_{\genparam}^{\KEG^n} } \hat{R}_k(\KEG_n, \genparam))_{n\in\NNReals}$ is a supermartingale with respect to the filtration $(\mathcal{F}_n)_{n\in\NNReals}$.
Moreover, by assumption, the loss is bounded and thus so also is the empirical risk. 
Supermartingale convergence then establishes that $\min_{\genparam\in \Omega_{\genparam}^{\KEG^n} } \hat{R}_k(\KEG_n, \genparam))$ converges almost surely and in $L_1$ to some random variable
that is measureable with respect to $\mathcal{F}_\infty$. 
The proof is completed by the fact that $\mathcal{F}_\infty$ is trivial.%
\end{proof}

\section{Basic asymptotics for random-walk sampling}
\label{sec:randomwalk}

In this section we establish the convergence of the relational empirical risk defined by the random walk.
The argument proceeds as follows:
We first recast the subsampling algorithm as a random probability measure, measurable with respect to the dataset graph $\KEG_n$.
Producing a graph according to the sampling algorithm is the same as drawing a graph according to the random measure.
Establishing that the relational empirical risk converges then amounts to establishing that expectations with respect to this random measure converge;
this is the content of \cref{thm:conv_of_RERM}. 
To establish this result, we show in \cref{lem:rw_simp_rand_meas} that sampling from the random-walk random measure is asymptotically equivalent 
to a simpler sampling procedure that depends only on the properties of the graphex process and not on the details of the dataset.
We allow for very general negative sampling distributions in this result; we show that how to specialize to the important case of 
(a power of) the unigram distribution in \cref{lem:rw_unigram_ns}.
  
\subsection{Random-walk Notation}
\label{ssec:rw_notation}
\newcommand{\walk}{H}

We begin with a formal description of the subsampling procedure that defines the relational empirical risk. 
We will work with random subset of the Poisson process $\PP$; these translate to random subgraphs of $\KEG$ in the obvious way. Namely,
if the sampler selects $\ppoint_i = (\glab_i,\gloc_i)$ in the Poisson process, then it selects $\ppoint_i$ in $\KEG$.

Sampling follows a two stage procedure: we choose a random walk, and then augment this random walk with additional vertices---this is the negative-sampling step. 
The following introduces much of the additional notation we require for this section.
\begin{defn}[Random-walk sampler]
Let $\mu_n$ be a (random) probability measure over $V(\KEG_n)$.
Let $\walk = (\ppoint_i)_{i\le M}= (\glab(\ppoint_i),\lambda(\ppoint_i))_{i\le M}$ be a sequence of vertices sampled according to:
\begin{enumerate}
\item (random-walk) $\ppoint_1 \sim \frac{d_n(\ppoint_1)}{2E_n}$ and let $\ppoint_i|\ppoint_{i-1}\sim \rm{unif}(\mathcal{N}_n(\ppoint_{i-1}))$ for $i\in (2,\dots,r+1)$. 
\item (augmentation) $\ppoint_{ r+2:M}$  be a sequence of additional vertices sampled from $\mu_n$ independently from each other and also from $(\ppoint_1,\dots,\ppoint_{r+1})$.  
\end{enumerate}
Let $G_{\walk}$ be the vertex induced subgraph of $\KEG_n$. 
Let $P_n = \Pr(G_{\walk} \in \cdot \given \bar{\Pi}_n(\bar{\genparam})) $ be the random probability distribution over subgraphs induced by this sampling scheme.
\end{defn}

With this notation in hand, We rewrite the loss function and the risk in a mathematically convenient form
\begin{defn}[Loss and risk]
The loss on a subsample is 
\[
L(G_{\walk}, \bar{\genparam})\in [0,1],
\] 
where we implicity restrict to the embeddings (and global parameters) associated with vertices in $G_{\walk}$. 
The empirical risk is 
\[
\EE_{P_n}[L(G_{\walk},\bar{\genparam})|\bar{\Pi}_n(\bar{\genparam})].
\]
\end{defn}

\begin{remark}
Note that the subgraphs produced by the sampling algorithm explicitly include all edges and non-edges of the graph.
However, the loss may (and generally will) depend on only a subset of the pairs. 
In this fashion, we allow for the practically necessary division between negative and positive examples.
Skipgram augmentation can be handled with the same strategy.    
\end{remark}

We impose a technical condition on the distribution that the additional vertices are drawn from. 
Intuitively, the condition is that the distribution is not too sensitive to details of the dataset in the large data limit.
\begin{defn}[Augmentation distribution]
\label{defn:augmentation_distribution}
We say $\mu_n$ is an \defnphrase{asymptotically exchangeable augmentation distribution} if is there 
is a $\mu$ such that
\begin{itemize}
\item There is a deterministic function $f$ such that $\mu(\eta)=f(\gloc(\eta))$%
\item $\|\mu_n(\cdot)-\frac{\mu(\cdot) \mathbb{I}(\cdot \in \KEG_n)}{n~Z_n}\|_{TV}\convPr{0},$ where $Z_n\triangleq  \frac{1}{n}\sum_{\eta\in \PP_n}{\mu(\eta)}.$
\end{itemize}
\end{defn} 
\cref{lem:rw_unigram_ns} establishes that the unigram distribution respects these conditions.

\subsection{Technical lemmas}

We begin with some technical inequalities controlling sums over the latent Poisson process.
To interpret the theorem, note that the degree of a vertex with latent property $\glocc$ is given by $f_n(\glocc,\PP)$ in the theorem statement.
\begin{lemma}\label{toronto}
\label{lem:pp_sum_ineqs}
Let  $(U_{\gloc(\eta)})_{\eta\in \Pi}$ be such that $(U_{\gloc(\eta)})_{\eta\in \Pi}|\Pi$ is distributed as a process of independent uniforms in $[0,1]$ and let 
\[
f_n(\glocc,\PP)\triangleq \sum_{\ppoint\in \PP_n}\mathbb{I}(U_{\gloc(\eta)}\le W(\glocc,\gloc)),
\]
for all $y\in\NNReals$.
Then the following hold:
\begin{enumerate}
\item
$\forall y\in\NNReals$ such that $W(\glocc,\cdot)\ge n^{-1+\frac{\epsilon}{4}}$, 
there are $p, K > 0$ such that $\forall \beta > 0$,
\[
\Pr(\big|\frac{  f_n(\glocc,\PP)}{n W(\glocc,\cdot)}-1\big|\ge \beta)\le \frac{K}{n^3 \beta^p}.
\] 
\item $\forall p>0 ,~\exists  K_p$ such that $\forall \beta>0$
\[
\Pr(\big|\frac{  f_n(\glocc,\PP))}{n }-W(\glocc,\cdot)\big|\ge \beta)\le \frac{K_p}{n^p \beta^{2p}}
\] and 
\[
\Pr(\big|\frac{ E_n}{n^2\mathcal {E}}-1\big|\ge \beta)\le \frac{K_p}{n^p \beta^{2p}}.
\]
\item $\exists K\in \NNReals$ such that  $\forall \glocc \in \NNReals$  such that $W(\glocc,\cdot)\le n^{-1+\frac{\epsilon}{4}}$  then 
$\Pr(f_n(\PP,\glocc) \ge n^{\frac{\epsilon}{2}})\le \frac{K}{n^3}.$
 
\end{enumerate}

\end{lemma}
\begin{proof}
We will first write the proof of the first statement, which is harder. We then highlight the differences in the other cases.
We use the Stein exchangeable pair method, recalled in \cref{exchangeable_pairs}.

Let $F:\mathbb{R}^2\rightarrow \mathbb{R}$ be such that $$\forall x,y~F(x,y)=[x-y].$$

Let $\bar{J}\sim \rm{unif}(\{0,n-1\})$ and let 
\[
\PP'=T_{[\bar{J},\bar{J}+1],[n,n+1]}\cdot \PP_{\glab}\times \PP_{\gloc},
\]
where $T_{[\bar{J},\bar{J}+1],[n,n+1]} $ is the permutation of  $[\bar{J},\bar{J}+1]$ and $[n,n+1]$ and 
\[
T_{[\bar{J},\bar{J}+1],[n,n+1]}\cdot \PP_{\glab}\times \PP_{\gloc}\triangleq \{(T_{[\bar{J},\bar{J}+1],[n,n+1]} (\glab),\gloc),~\forall (\glab,\gloc)\in \PP\}
\]

Then we can check the following:
\begin{itemize}
\item As $\PP\cap [0,n]\setminus [\bar{j},\bar{j}+1]\times \mathbb{R}^+= \PP'\cap[0,n]\setminus [\bar{j},\bar{j}+1]\times \mathbb{R}^+$  we obtain that  \begin{equation*}\begin{split}&\EE(\frac{f_n(\glocc,\PP)}{W(\glocc,\cdot)}-\frac{f_n(\glocc,\PP')}{W(\glocc,\cdot)}\big|\PP_n)\\&\overset{(a)}{=}\frac{1}{nW(\glocc,\cdot)} [\sum_{j=0}^{n-1}\sum_{\PP_{j+1}\setminus \PP_j}\mathbb{I}(U_{\gloc(\eta)}\le W(\glocc,\gloc))- \EE(\mathbb{I}(U_{\gloc(\eta)}\le W(\glocc,\gloc)))] 
\\&\overset{(b)}{=}\frac{ f_n(\glocc,\PP)}{nW(\glocc,\cdot)}-1 \end{split}\end{equation*}
where (a) is obtained by complete independence of $\PP$ and where to get (b) we use the fact that (see \cite{Veitch:Roy:2015})
\[
\sum_{(\glab,\gloc)\in \PP_{j+1}\setminus \PP_j}\mathbb{I}(U_{\gloc(\eta)}\le W(\glocc,\gloc))\sim \poiDist(W(\glocc,\cdot))
\]

\item Moreover, we can very similarly see that:
{\begin{equation*}\begin{split}&\Big\|\frac{1}{2n}\EE([\frac{f_n(\glocc,\PP)}{W(\glocc,\cdot)}-\frac{f_n(\glocc,\PP')}{W(\glocc,\cdot)}]^2\big|\PP_n)\Big\|_{p}
\\&\le\frac{1}{n^2W(\glocc,\cdot)^2}\Big\| \sum_{j=0}^{n-1}\big[[\sum_{(\glab,\gloc)\in \PP_{j+1}\setminus \PP_j}\mathbb{I}(U_{\gloc(\eta)}\le W(\glocc,\gloc))]^2+2 W(\glocc,\cdot)\big]\Big\|_p
\\& \le \frac{1}{n^2W(\glocc,\cdot)^2} \sum_{j=0}^{n-1}\big\|[\sum_{(\glab,\gloc)\in \PP_{j+1}\setminus \PP_j}\mathbb{I}(U_{\gloc(\eta)}\le W(\glocc,\gloc))]^2\|_p+2 W(\glocc,\cdot)
\\& \le \frac{C}{nW(\glocc,\cdot)},
\end{split}\end{equation*}}
where $C$ is a constant that does not depend on $n$ or $\glocc$.
\end{itemize}

Therefore using the exchangeable pair method presented earlier and setting $p\ge \frac{12}{\epsilon}$ for all $\glocc $ such that $W(\glocc,\cdot)\ge n^{\frac{\epsilon}{4}-1}$ we get that there is $K$,$p$ such that  for all $\epsilon>0$
$$P(|\frac{\sum_{(\glab,\gloc)\in \PP_n}\mathbb{I}(U_{\gloc(\eta)}\le W(\glocc,\gloc))}{ W(\glocc,\cdot)}-1|\ge \beta) \le \frac{K}{n^3 \beta^p},$$
QED. 

For the second statement, instead of $\frac{f_n(\glocc,\PP)}{W(\glocc,\cdot)}$ we are interested in $f_n(\glocc, \PP)$, which is easier to handle. 
Indeed, using the same exchangeable pair $(\PP,\PP')$ we get that:
\begin{itemize}
\item As $\PP\cap [0,n]\setminus [\bar{j},\bar{j}+1]\times \mathbb{R}^+= \PP'\cap[0,n]\setminus [\bar{j},\bar{j}+1]\times \mathbb{R}^+$  we obtain that  \begin{equation*}\begin{split}&\EE(f_n(\glocc,\PP)-f_n(\glocc,\PP')\big|\PP_n)
\\&=\frac{1}{n}f_n(\glocc,\PP)-W(\glocc,\cdot) .\end{split}\end{equation*}

\item Moreover we can very similarly see that:{\begin{equation*}\begin{split}&\Big\|\frac{1}{2n}\EE([f_n(\glocc,\PP)-f_n(\glocc,\PP')]^2\big|\PP_n)\Big\|_{p}
\\& \le \frac{1}{n^2} \sum_{j=0}^{n-1}\big\|[\sum_{(\glab,\gloc)\in \PP_{j+1}\setminus \PP_j}\mathbb{I}(U_{\gloc(\eta)}\le W(\glocc,\gloc))]^2\|_p+2 W(\glocc,\cdot)
\\& \le \frac{C}{n},
\end{split}\end{equation*}}
where $C$ is a constant that does not depend on $n$ or $\glocc$.
Therefore we get the desired result QED. 

A very similar roadmap can be followed for $E_n$.

The last statement is a simple consequence of the preceding results. Indeed, for all $\glocc \in \mathbb{R}$,
\[
P(W(\glocc)\le n^{-1+\frac{\epsilon}{4}}, f_n(\PP,\glocc) \ge n^{\frac{\epsilon}{2}})\le P( |\frac{f_n(\PP,\glocc)}{n}-W(\glocc,\cdot)|\ge n^{-\frac{\epsilon}{4}}) \le\frac{K_{\frac{3}{1+\frac{\epsilon}{4}}}}{n^3}.
\]
\end{itemize}
\end{proof}

With this in hand, we establish the asymptotic equivalence of random-walk sampling and a sampling scheme
that does not depend on the details of the dataset. This is the main component of the proof.
Recall the notation introduced in \cref{ssec:rw_notation}.
\begin{lemma} 
\label{lem:rw_simp_rand_meas}
Suppose that there is $\epsilon\in (0,1)$ such that the graphon $W$ verifies
\[
W(\gloc,\cdot)=O(\gloc^{-1-\epsilon}).
\]
Suppose further that the augmented sampling distributions $(\mu_n)_n$ satisfy the conditions of \cref{defn:augmentation_distribution}. 
Then, writing
\[
P_n(\walk)\triangleq \mathbb{I}(\ppoint_{1:r+1}\in {\paths}_r(\PP_n)) \frac{\prod_{l=r+2}^{M}\mu_n(\ppoint_l)}{2N^n_e \prod_{i=2}^r d_n(\ppoint_i)}$$ and $$\tilde{P}_n(\walk)\triangleq \mathbb{I}(\ppoint_{1:r+1}\in {\paths}_r(\PP_n)) \frac{\prod_{l=r+2}^{M}\mu(\ppoint_l)}{2n^{M}\mathcal{E}\prod_{i=2}^rW(\gloc(\ppoint_i),\cdot)},
\]
it holds that
{\begin{equation*}\begin{split}&
\sup_{\bar{\genparam}\in \Omega_{\genparam}^{\PP}}\Big|\EE_{P_n}\big(L(G_{\walk},\bar{\genparam})|\bar{\Pi}_n(\bar{\genparam})\big)-\EE_{\tilde{P}_n}\big(L(G_{\walk},\bar{\genparam})|\bar{\Pi}_n(\bar{\genparam})\big)\Big|=o_p(1).
\end{split}\end{equation*}}

\end{lemma}
\begin{proof} %

We can first see by the triangle inequality that if we write the following two measures: $$P_n^*(\walk)\triangleq \mathbb{I}(\ppoint_{1:r+1}\in {\paths}_r(\PP_n)) \frac{\prod_{l=r+2}^{M}\mu(\ppoint_l)}{2N^n_e n^{M-(r+1)}  \prod_{i=2}^r d_n(\ppoint_i)}$$  and 
$$\tilde{P}_n^*(\walk)\triangleq\mathbb{I}(\ppoint_{1:r+1}\in {\paths}_r(\PP_n)) \frac{\prod_{i=2}^r \mathbb{I}( W(\gloc(\ppoint_i),\cdot)\ge n^{-1+\frac{\epsilon}{4}})\prod_{l=r+2}^{M}\mu(\ppoint_l)}{2n^{M}\mathcal{E}\prod_{i=2}^rW(\gloc(\ppoint_i),\cdot)}$$
Then $\forall \beta>0$:
\begin{equation*}\begin{split}&
P\Big(\sup_{\bar{\genparam}\in \Omega_{\genparam}^{\PP}}\Big|\EE_{P_n}\big(L(G_{\walk},\bar{\genparam})|
\bar{\Pi}_n(\bar{\genparam})\big)-\EE_{\tilde{P}_n}\big(L(G_{\walk},\bar{\genparam})|
\bar{\Pi}_n(\bar{\genparam})\big)\Big|>\beta)
\\& \le P\Big(\sup_{\bar{\genparam}\in \Omega_{\genparam}^{\PP}}\Big|\EE_{P_n}\big(L(G_{\walk},\bar{\genparam})|
\bar{\Pi}_n(\bar{\genparam})\big)-\EE_{P_n^*}\big(L(G_{\walk},\bar{\genparam})|
\bar{\Pi}_n(\bar{\genparam})\big)\Big|>\frac{\beta}{3}) 
\\&+P\Big(\sup_{\bar{\genparam}\in \Omega_{\genparam}^{\PP}}\Big|\EE_{P^*_n}\big(L(G_{\walk},\bar{\genparam})|
\bar{\Pi}_n(\bar{\genparam})\big)-\EE_{\tilde{P}_n^*}\big(L(G_{\walk},\bar{\genparam})|
\bar{\Pi}_n(\bar{\genparam})\big)\Big|>\frac{\beta}{3})
\\& +P\Big(\sup_{\bar{\genparam}\in \Omega_{\genparam}^{\PP}}\Big|\EE_{\tilde{P}_n^*}\big(L(G_{\walk},\bar{\genparam})|
\bar{\Pi}_n(\bar{\genparam})\big)-\EE_{\tilde{P}_n}\big(L(G_{\walk},\bar{\genparam})|
\bar{\Pi}_n(\bar{\genparam})\big)\Big|>\frac{\beta}{3}),
\end{split}\end{equation*}
therefore proving that the last terms converge to zero for any $\beta>0$ is sufficient.

First we will prove that
{ \begin{equation*}\begin{split}&
\sup_{\bar{\genparam}\in \Omega_{\genparam}^{\PP}}\Big|\EE_{P_n}\big(L(G_{\walk},\bar{\genparam})|
\bar{\Pi}_n(\bar{\genparam})\big)-\EE_{P_n^*}\big(L(G_{\walk},\bar{\genparam})|
\bar{\Pi}_n(\bar{\genparam})\big)\Big|=o_p(1).
\end{split}\end{equation*}}

Indeed, noting that, 
\[
P_{n,i}^*(\walk)\triangleq \mathbb{I}(\ppoint_{1:r+1}\in {\paths}_r(\PP_n)) \frac{\prod_{l=r+2}^{r+1+i}\mu(\ppoint_l) ~\prod_{r+2+i}^{M}\mu_n(\ppoint_l) }{2E_n n^i \prod_{i=2}^r d_n(\ppoint_i)},
\]
it holds $\forall\beta>0$ that
{ \begin{equation*}\begin{split}&
P\big(\sup_{\bar{\genparam}\in \Omega_{\genparam}^{\PP}}\Big|\EE_{P_n}\big(L(G_{\walk},\bar{\genparam})|
\bar{\Pi}_n(\bar{\genparam})\big)-\EE_{P_n^*}\big(L(G_{\walk},\bar{\genparam})|
\bar{\Pi}_n(\bar{\genparam})\big)\Big|>\beta\big)
\\& \overset{(a)}{\le} \sum_{i=1}^{M} P\big(\sup_{\bar{\genparam}\in \Omega_{\genparam}^{\PP}}\big|\EE_{P^*_{n,i}}\big(L(G_{\walk},\bar{\genparam})\big)-\EE_{P_{n,i-1}^*}\big(L(G_{\walk},\bar{\genparam})\big)\big|>\frac{\beta}{M}\big)
\\& \le M P\big( \|\mu_n-\frac{\mu}{nZ_{\mu}}\|_{TV}>\frac{\beta}{\|L\|_{\infty}}\big).
\end{split}\end{equation*}}
where (a) using telescopic sum.
Therefore we have proven that the first element of the sum goes to 0.

Now we will prove that 
\[
\sup_{\bar{\genparam}\in \Omega_{\genparam}^{\PP}}
\abs{ \EE_{P^*_n}\big(L(G_{\walk},\bar{\genparam}) \given  \bar{\Pi}_n(\bar{\genparam})\big)
-\EE_{\tilde{P}_n^*}\big(L(G_{\walk},\bar{\genparam})| \bar{\Pi}_n(\bar{\genparam})\big)} =o_p(1).
\]
For this we will  want to approximate $\frac{n}{d_n(V_{u_i})}$ by $\frac{1}{W(u_i,\cdot)}$.
However for this we need a good bound on $P(|\frac{d_n(V_{u_i})}{nW(u_i,\cdot)}-1|\ge \epsilon)$. 
But this is possible only if $W(u_i,\cdot)$ is not too small.

Note that for all vertices  $ \ppoint\in \PP_n$  if a path $\walk$ passes through $\ppoint$ at the $i$-th coordinate,  for $i\ge 2$, then it means that there is only  $d_n(\glab(\ppoint))$ possibilities for the $i-1$th vertex of the path. Therefore if $d_n(\glab(\ppoint))$ is small the probability that our random-walk passes through $v$, and is not the origin vertex, is asymptotically negligible.%

 Indeed for all  $\ppoint\in \PP_n$ such that $d_n(\glab(\ppoint))\le n^{\frac{\epsilon}{2}}$ it holds that for $k\ge 2$, 
\[
P(\ppoint_i=\ppoint \given \bar{\Pi}_n(\bar{\genparam}))\le 
\sum_{\ppoint'\in  \PP_n\cap \mathcal{N}_n(\ppoint)}P(\ppoint_{i-1}=\ppoint',\ppoint_i=\ppoint \given \bar{\Pi}_n(\bar{\genparam}))\overset{(*)}{\le}\frac{ n^{\frac{\epsilon}{2}}}{2N_n^e},
\]
where to get (*) we used the stationary property of the random walk.

Therefore we have:
 \begin{equation*}\begin{split}P(\min_{k\ge 2}d_n(\ppoint_k)\le n^{-\frac{\epsilon}{2}} \given
\bar{\Pi}_n(\bar{\genparam}))\le\frac{ rn^{\frac{\epsilon}{2}}\big|\{\ppoint\in \PP_n,~\mathrm{ s.t. }~0<d_n(\ppoint)\le n^{\frac{\epsilon}{2}}\}\big|}{2N_n^e}\convPr 0,\end{split}\end{equation*}
But we have that $\forall (\ppoint_i)_{i\le r+1}$ s.t. $\forall i, W(\gloc(\ppoint_i),\cdot)\ge n^{-1+\frac{\epsilon}{4}}$,
{\small \begin{equation*}\begin{split}&
\Big|\frac{1}{2E_n \prod_{i=2}^r d_n(\ppoint_i)}-\frac{1}{2n^{r+1}\mathcal{E}\prod_{i=2}^rW(\gloc(\ppoint_i),\cdot)}\Big|
\\&\overset{(a)}{ \le} \sum_{i=2}^{r}\frac{1}{2E_n n^{i-1}\prod_{l=2}^{r-i} d_n(\ppoint_l)\prod_{l=r-i+2}^rW(\gloc(\ppoint_l),\cdot)}\Big|\frac{1}{ d_n(\ppoint_{r-i+1})}-\frac{1}{n^{}W(\gloc(\ppoint_{r-i+1}),\cdot)}\Big|  \\&~~~+ \frac{1}{n^{r-1}\prod_{l=2}^{r} W(\gloc(\ppoint_l),\cdot)}\Big|\frac{1}{2N_e^n}-\frac{1}{2n^2\mathcal{E}}\Big|
\\& \le \sum_{i=2}^{r}\frac{1}{2E_n n^{i-1}\prod_{l=2}^{r-i+1} d_n(\ppoint_l)\prod_{l=r-i+2}^rW(\gloc(\ppoint_l),\cdot)}\Big|1-\frac{d_n(\ppoint_{r-i+1})}{n^{}W(\gloc(\ppoint_{r-i+1}),\cdot)}\Big|  + \frac{1}{2n^{r-1}N_e^n\prod_{l=2}^{r} W(\gloc(\ppoint_l),\cdot)}\Big|1-\frac{N_e^n}{n^2\mathcal{E}}\Big|,
\end{split}\end{equation*}} 
where (a) comes from a simple telescopic sum re-writing.

Therefore if $$\max_i \Big|1-\frac{d_n(\glab_{i})}{n^{}W(\glocc_{i},\cdot)}\Big|, \Big|1-\frac{N_e^n}{n^2\mathcal{E}}\Big|\le \beta$$ then 
{\small \begin{equation*}\begin{split}&
\Big|\frac{1}{2E_n \prod_{i=2}^r d_n(\ppoint_i)}-\frac{1}{2n^{r+1}\mathcal{E}\prod_{i=1}^rW(\gloc(\ppoint_i),\cdot)}\Big| 
\\& \le \beta \big[\sum_{i=2}^{r}\frac{1}{2E_n n^{i-1}\prod_{l=2}^{r-i+1} d_n(\ppoint_l)\prod_{l=r-i+2}^rW(\gloc(\ppoint_l),\cdot)}  + \frac{1}{2n^{r-1}N_e^n\prod_{l=2}^{r} W(\gloc(\ppoint_l),\cdot)}\big]
\end{split}\end{equation*}} 

Now note that
 for all $i$, and $\embedding'\in \Omega$ %
{\small \begin{equation*}\begin{split}&
\sum_{\ppoint_{1:r+1}\in {\paths}_r(\PP_n)} \frac{ \prod_{i=2}^r\mathbb{I}( W(\gloc(\ppoint_i),\cdot)\ge n^{-1+\frac{\epsilon}{4}})}{2 E_n n^{i-1}\prod_{l=2}^{r-i+1} d_n(\ppoint_l)\prod_{l=r-i+2}^rW(\gloc(\ppoint_l),\cdot)}  \EE\big(L(G_{\walk},\bar{\genparam})|\ppoint_{r+2:M_n},\PP_n\big)
\\&\overset{(a)}{ \le}\sum_{\ppoint_{1:r}\in  {\paths}_{r-1}(\PP_n)} d_n(\ppoint_r)  \frac{ \prod_{i=2}^r\mathbb{I}( W(\gloc(\ppoint_i),\cdot)\ge n^{-1+\frac{\epsilon}{4}})}{2 E_n n^{i-1}\prod_{l=2}^{r-i+1} d_n(\ppoint_l)\prod_{l=r-i+2}^rW(\gloc(\ppoint_l),\cdot)}  \EE\big(L(G_{\walk},\bar{\genparam})|\ppoint_{r+2:M_n},\PP_n\big)
\\&\le \|L\|_{\infty} \max_{\glocc \in N^n_v(\PP) ~~\mathrm{ s.t. }~W(\glocc,\cdot)\ge n^{-1+\frac{\epsilon}{4}}} \frac{d_n(\glocc)}{nW(\glocc,\cdot)}\sum_{\ppoint_{1:r}\in  {\paths}_{r-1}(\PP_n)}  \frac{ \prod_{i=2}^r\mathbb{I}( W(\gloc(\ppoint_i),\cdot)\ge n^{-1+\frac{\epsilon}{4}})}{2 E_n n^{i-1}\prod_{l=2}^{r-i+1} d_n(\ppoint_l)\prod_{l=r-i+2}^{r-1}W(\gloc(\ppoint_l),\cdot)}  
\end{split}\end{equation*}} 
where (a)  is a simple consequence form the fact that:
$$\rm{card}\{\ppoint\in \ppoint(\PP_n,r)~\mathrm{ s.t. }~\ppoint|_{1:r}=(\glab_i,\glocc_i)_{1:r} \}= d_n(\glab_r)\rm{card}\{\ppoint\in \ppoint(\PP_n,r-1)~\mathrm{ s.t. }~\ppoint|_{1:r-1}=(\glab_i,\glocc_i)_{1:r-1} \}.$$

Therefore, by induction, we can get that for all $i$
{\small \begin{equation*}\begin{split}&
\sum_{\ppoint_{1:r+1}\in {\paths}_r(\PP_n)} \prod_{i=2}^r \mathbb{I}( W(\gloc(\ppoint_i),\cdot)\ge n^{-1+\frac{\epsilon}{4}})\frac{\EE\big(L(G_{\walk},\bar{\genparam})|\ppoint_{r+2:M},\PP_n\big)}{E_n n^{i-1}\prod_{l=2}^{r-i+1} d_n(\ppoint_l)\prod_{l=r-i+2}^rW(\gloc(\ppoint_l),\cdot)}  
\\& \le r\|L\|_{\infty}\max_{\glocc \in N^n_v(\glocc) ~~\mathrm{ s.t. }~W(\glocc,\cdot)\ge n^{-1+\frac{\epsilon}{4}}}|\frac{d_n(\glocc)}{nW(\glocc,\cdot)}-1|+\|L\|_{\infty}.
\end{split}\end{equation*}} 

Therefore if we note $$A_n(\beta)\triangleq \{ \max_{\glocc \in N^n_v(\glocc) ~~\mathrm{ s.t. }~W(\glocc,\cdot)\ge n^{-1+\frac{\epsilon}{4}}}|\frac{d_n(\glocc)}{nW(\glocc,\cdot)}-1|\le \beta, |\frac{N_e^n}{n^2\mathcal{E}}-1|\le \beta\}$$

Then we can see the following:
\begin{itemize}
\item On $A_n(\beta)$ we will have that as $\ppoint_{1:r+1}\independent \ppoint_{r+2:M}$ using the result that we previously got we have that:
$$\sup_{\bar{\genparam}\in \Omega_{\genparam}^{\PP}}\Big|\EE_{P^*_n}\big(L(G_{\walk},\bar{\genparam})|
\bar{\Pi}_n(\bar{\genparam})\big)-\EE_{\tilde{P}_n^*}\big(L(G_{\walk},\bar{\genparam})|
\bar{\Pi}_n(\bar{\genparam})\big)\Big|\le (r+1)^2\|L\|_{\infty}\beta$$
\item And in addition we know that there is $K_1,K_2<\infty$ s.t
{\small \begin{equation*}\begin{split}&
P(A_n(\beta)^c) \le P(|\frac{N_e^n}{n^2\mathcal{E}}-1|\ge \beta)+ \EE\big(\sum_{\ppoint_{1:r+1}\in {\paths}_r(\PP_n)}\mathbb{I}(|\frac{d_n(\glocc)}{nW(\glocc,\cdot)}-1|\ge \beta )\big)
\\& \overset{(a)}{\le}  P(|\frac{N_e^n}{n^2\mathcal{E}}-1|\ge \beta)+ n\int_{\mathbb{R}^+} \mathbb{I}(W(\gloc,\cdot) \ge n^{-1+\frac{\epsilon}{4}})P\big(|\frac{f_n(\gloc, \PP)}{nW(\gloc,\cdot)}-1|\ge \beta\big) \big)d\gloc
\\& \overset{(b)}{\le} \frac{K_1}{n \beta}+ \frac{K_2}{\beta^pn^2}\int_{\mathbb{R}^+} \mathbb{I}(W(\gloc,\cdot) \ge n^{-1+\frac{\epsilon}{4}})d\gloc
\\& \le  \frac{K_1}{n \beta}+ \frac{K_2}{\beta^pn^2}n^{1-\frac{3\epsilon}{2+2\epsilon}}\rightarrow 0,
\end{split}\end{equation*}} 
where (a) comes from Slivnyak–Mecke theorem and (b) from \cref{toronto}.

Thus, we have successfully proven that 
{\small \begin{equation*}\begin{split}&
\sup_{\bar{\genparam}\in \Omega_{\genparam}^{\PP}}\Big| \EE_{P^*_n}\big(L(G_{\walk},\bar{\genparam})|
\bar{\Pi}_n(\bar{\genparam})\big)-\EE_{\tilde{P}_n^*}\big(L(G_{\walk},\bar{\genparam})|
\bar{\Pi}_n(\bar{\genparam})\big)\Big|=o_p(1)
\end{split}\end{equation*}}
QED\end{itemize}

Now we are going to prove the last part, i.e.
{ \begin{equation*}\begin{split}&
\sup_{\bar{\genparam}\in \Omega_{\genparam}^{\PP}} \Big|\EE_{\tilde{P}_n^*}\big(L(G_{\walk},\bar{\genparam})|
\bar{\Pi}_n(\bar{\genparam})\big)-\EE_{\tilde{P}_n}\big(L(G_{\walk},\bar{\genparam})|
\bar{\Pi}_n(\bar{\genparam})\big)\Big|=o_p(1)
\end{split}\end{equation*}}
 For this we can note that that for all $i\ge 2$
\small \begin{equation*}\begin{split}&
\|\frac{1}{n^{r+1}}\sup_{\embedding'\in \Omega_{\genparam}^{\PP}}\sum_{\ppoint_{1:r+1}\in {\paths}_r(\PP_n)}\frac{\mathbb{I}( W(\gloc(\ppoint_i),\cdot)< n^{-1+\frac{\epsilon}{4}})}{2n^{r+1}\mathcal{E}\prod_{i=2}^rW(\gloc(\ppoint_i),\cdot)}\EE\big(L(G_{\walk},\bar{\genparam})|
\bar{\Pi}_n(\bar{\genparam}),\ppoint_{r+2,M}\big)\|_{L_1}
\\& \overset{(a)}{\le} \|L\|_{\infty} \int_{\mathbb{R}^{r+1}} \mathbb{I}( W(\gloc(\ppoint_i),\cdot)< n^{-1+\frac{\epsilon}{4}})\frac{\prod_{j=1}^{r}W(\gloc_j,\gloc_{j+1})}{\prod_{j=2}^rW(\gloc_j,\cdot)}d\gloc_{1:r+1}
\\& \overset{(b)}{\le} \|L\|_{\infty}\int_{\mathbb{R}^i}\mathbb{I}( W(\gloc(\ppoint_i),\cdot)< n^{-1+\frac{\epsilon}{4}})\frac{\prod_{j=1}^{i-1}W(\gloc_j,\gloc_{j+1}) }{\prod_{j=2}^{i-1}W(\gloc_j,\cdot)}d\gloc_{1:i}
\\& \overset{(c)}{\le }\|L\|_{\infty}\int_{\mathbb{R}} W(\gloc(\ppoint_i),\cdot)\mathbb{I}( W(\gloc(\ppoint_i),\cdot)< n^{-1+\frac{\epsilon}{4}}) d\gloc_i \xrightarrow{n\rightarrow \infty} 0,
\end{split}\end{equation*}
where to get (a)  we used both the fact that $L$ was bounded and  the independence of the uniforms; to get (b) we integrated coordinates $r+1$ to $i+1$ and used the following definition:
\[
\forall\gloc ~\int W(\gloc',\gloc)d\gloc'=W(\gloc,\cdot).
\]
We similarly got (c) where instead we integrated the coordinates from 1 to $i-1$.

Therefore we have  successfully proven that
{ \begin{equation*}\begin{split}&
\sup_{\bar{\genparam}\in \Omega_{\genparam}^{\PP}} \Big|\EE_{\tilde{P}_n^*}\big(L(G_{\walk},\bar{\genparam})|
\bar{\Pi}_n(\bar{\genparam})\big)-\EE_{\tilde{P}_n}\big(L(G_{\walk},\bar{\genparam})|
\bar{\Pi}_n(\bar{\genparam})\big)\Big|=o_p(1)
\end{split}\end{equation*}}

Hence we have proven the desired results
\end{proof}

We now turn to the question of which augmentation distributions will satisfy the conditions of the previous result.
We show that the conditions hold for any distribution defined by a differentiable function of the unigram distribution;
in particular, this covers the unigram distribution to the power of $\nicefrac{3}{4}$ that is used to define unigram negative sampling.  
\newcommand{\unigDist}{\mathrm{Ug}}
\begin{lemma}
\label{lem:rw_unigram_ns}
Let $\ppoint_{1: r+1}$ be sampled by a random walk on $G_n$, and let the random-walk unigram distribution be defined by
\[
\unigDist_{\KEG_n}(\ppoint) = \Pr(\exists i\le r+1,~\mathrm{ s.t. }~\tilde{\ppoint_i}=\ppoint \given \bar{\Pi}_n(\bar{\embedding})).
\] 
Suppose that $\mu_n$ is defined by
\[
\mu_n(\ppoint) \propto \unigDist_{\KEG_n}(\ppoint)^{\alpha},
\] 
for a certain $\alpha>0$.
Then, defining $\mu$ by
\[
\mu(\ppoint)\propto(r+1)^{\alpha}\frac{W(\gloc,\cdot)^{\alpha}}{\mathcal{E}^{\alpha}},
\] 
it holds that
\[
\|\mu_n-\frac{\mu(\cdot) \mathbb{I}({\cdot\in \PP_n})}{nZ_n}\|_{TV} \convPr 0
\]
\end{lemma}
\begin{proof}
We will for simplicity prove the result for $\alpha=1$, the other cases can be obtained following a similarly, although the computations are more involved.

First, self-intersections of the walk are asymptotically negligible:
 \begin{equation*}\begin{split}& \sum_{\ppoint\in \PP_n}\big| P(\exists i\le r+1,~\mathrm{ s.t. }~\tilde{\ppoint_i}=\ppoint|  \bar{\Pi}_n(\bar{\embedding}) )-  \sum_{i=1}^{r+1}P(\tilde{\ppoint_i}=\ppoint| \bar{\Pi}_n(\bar{\embedding}) )\big|
\\&\overset{(a)}{ \le} \sum_{\ppoint\in \PP_n} \sum_{i=1}^{r+1}P(\tilde{\ppoint_i}=\ppoint|  \bar{\Pi}_n(\bar{\embedding}) )P(\exists j\in [i+1, r+1],~\ppoint_j=\ppoint|\ppoint_i=\ppoint,  \bar{\Pi}_n(\bar{\embedding}) )\xrightarrow{P,(b)}0, 
\end{split}\end{equation*}

where (b) comes from the dominated convergence theorem and (a) comes from the fact that for all $\eta$ { \begin{equation*}\begin{split}&\Big|\EE\big(\mathbb{I}\big(\exists i\le r+1,~\mathrm{ s.t. }~\tilde{\ppoint_i}=\ppoint\big)-\sum_{i=1}^{r+1}\mathbb{I}(\tilde{\ppoint_i}=\ppoint)|  \bar{\Pi}_n(\bar{\embedding}) \big)\Big|
\\&\le \sum_{i=1}^{r+1}\EE\big(\mathbb{I}(\tilde{\ppoint}_i=\ppoint,~\exists j\ge i~\mathrm{ s.t. }~\tilde{\ppoint}_j=\ppoint)| \KEG_n\big)\end{split}\end{equation*}}

Next, the limiting probability that a walk includes $\ppoint$ is determined by its limiting relative degree $\frac{W(\gloc(\ppoint),\cdot)}{2\mathcal{E}}$.
To that end, we write:
\begin{equation*}\begin{split}& \sum_{\ppoint\in \PP_n}\big|  \sum_{i=1}^{r+1}P(\tilde{\ppoint_i}=\ppoint|  \bar{\Pi}_n(\bar{\embedding}) )-\frac{(r+1)W(\gloc(\ppoint),\cdot)}{2n\mathcal{E}}\big|
\\&\overset{(a)}{ \le} \sum_{\ppoint\in \PP_n}\big|   \frac{(r+1)d_n(\ppoint)}{2E_n}-\frac{(r+1)W(\gloc(\ppoint),\cdot)}{2n\mathcal{E}}\big|
\end{split}\end{equation*}
where  (a) comes from the stationarity proprieties of the simple random walk.%
Then, using \cref{toronto}, we see that:
{ \begin{equation*}\begin{split}& \sum_{\ppoint\in \PP_n}\big|   \sum_{i=1}^{r+1}P(\tilde{\ppoint_i}=\ppoint| \bar{\Pi}_n(\bar{\embedding}))-\frac{(r+1)W(\gloc(\ppoint),\cdot)}{2n\mathcal{E}}\big|=o_p(1).
\end{split}\end{equation*}}

Finally,
{ \begin{equation*}\begin{split}& \sum_{\ppoint\in \PP_n}\big| \frac{(r+1)W(\gloc(\ppoint),\cdot)}{2n\mathcal{E}}[1-\frac{1}{\sum_{\ppoint\in \PP_n}\frac{(r+1)W(\gloc(\ppoint),\cdot)}{2n\mathcal{E}}}]\big|
\\& = \sum_{\ppoint\in \PP_n} \frac{(r+1)W(\gloc(\ppoint),\cdot)}{2n\mathcal{E}}-1
\\& =\sum_{\ppoint\in \PP_n} \frac{(r+1)W(\gloc(\ppoint),\cdot)}{2n\mathcal{E}}-P(\exists i\le r+1,~\mathrm{ s.t. }~\tilde{\ppoint_i}=\ppoint| \KEG_n) =o_p(1).
\end{split}\end{equation*}}

\end{proof}

\subsection{Convergence for random walk sampling}
Let $\bar{\genparam}$ be a random element of $\Omega_{\genparam}^{\Pi}$ such that  $\bar{\genparam}|\Pi\sim \mathcal{Q}_{\genparam}^{\Pi}$ for a certain kernel $m$.
For brevity, we write  
\[
\hat{R}_k(G_n,\bar{\genparam})\triangleq \EE_{P_n}\big(L(G_{\walk},\bar{\genparam})|
\bar{\Pi}_n(\bar{\genparam})\big).
\]
for all $n\in\NNReals$.
\begin{theorem} 
\label{thm:conv_of_RERM}
There are constants $c^{\mathrm{rw}}_{m}, c^{\mathrm{rw}}_{*} \in \NNReals$ such that 
\[
\hat{R}_k(G_n,\bar{\genparam})\convPr c^{\mathrm{rw}}_{m},
\]
 and
\[
\min_{\bar{\genparam}\in \Omega_w^{\PP}}\hat{R}_k(G_n,\bar{\genparam})\convPr c^{\mathrm{rw}}_{*}.
\]

And those constants are respectively $\lim_n\mathbb{E}(\hat{R}_k(G_n,\bar{\genparam}))$ and $\lim_n\mathbb{E}(\min_{\bar{\genparam}\in \Omega_w^{\PP}}\hat{R}_k(G_n,\bar{\genparam}))$ 
\end{theorem}

\begin{proof} 
\cref{lem:rw_simp_rand_meas}
states that 
\begin{itemize}
\item
$
\EE_{P_n}\big(L(G_{\walk},\bar{\genparam})|
\bar{\Pi}_n(\bar{\genparam})\big)-\EE_{\tilde{P}_n}\big(L(G_{\walk},\bar{\genparam})|
\bar{\Pi}_n(\bar{\genparam})\big)=o_p(1).
$
\item {\small $\min_{\bar{\genparam}\in \Omega_{\genparam}^{\PP}}
\EE_{P_n}\big(L(G_{\walk},G_{\walk}({\embedding}))|
\bar{\Pi}_n(\bar{\genparam})\big)-\min_{\bar{\genparam}\in \Omega_{\genparam}^{\PP}}\EE_{\tilde{P}_n}\big(L(G_{\walk},G_{\walk}({\embedding}))|
\bar{\Pi}_n(\bar{\genparam})\big)=o_p(1).$}
\end{itemize}We will see that $\mathbb{E}_{\tilde P_n}$ inherits much of the nice distributional structure of the point process $\Pi$.
This will be essential to the proof.

 To see this we first define for all integers $i\in \mathbb{N}$ the restriction of the point process to points $\eta\in \Pi$ such that $\nu(\eta)\in (i,i+1]$,$$\Pi|_{(i,i+1]}:=\Pi_{i+1}\setminus \Pi_{i}.$$  And for all M sequence of integers $ I=(I_1,\dots,I_{M})\in \mathbb{N}^{M}$ we write the following sequence of M restrictions of $\Pi$,$$ \PP|_{I}\triangleq (\Pi|_{(I_1,I_{1}+1]},\dots,\PP|_{(I_M,I_{M}+1]}).$$ This allows us to define the following M-dimensional array $X(\bar{\genparam})\triangleq (X_{I}(\bar{\genparam}))_{I\in \mathbb{N}^M}$ where for all M integers  $ I=(I_1,\dots,I_{M})\in \mathbb{N}^{M}$,
\[
X_{I}(\bar{\genparam})\triangleq \sum_{\ppoint_{1:M}\in \PP|_{I}} \frac{\mathbb{I}(\ppoint_{1:r+1}\in {\paths}_r(\PP_n))\prod_{l=r+2}^{M}\mu(\gloc(\ppoint_l))}{2\mathcal{E} \prod_{i=2}^r W(\gloc(\ppoint_i),\cdot)} L(G_{H}, G_{H}(\bar{\genparam})).
\]
This quantity is key as  we can write that \begin{equation}\begin{split}\label{rw_samp}
\EE_{\tilde{P}_n}\big(L(G_{\walk},\bar{\genparam})|
\bar{\Pi}_n(\bar{\genparam})\big)& =\frac{1}{n^M}\sum_{i_{1:M}\le n-1}X_{i_{1:M}}^{\bar{\genparam}}.
\end{split}\end{equation}
 Then {using classical results} on convergence of exchangeable arrays \cite{Kallenberg:1999} we obtain that:
\[
\EE_{P_n}\big(L(G_{\walk},\genparam)| \bar{\Pi}_n(\bar{\genparam})\big) \convPr 
\int_{\mathbb{R}_+^M}\mathcal{R}(\gloc_{1:M})\frac{\prod_{i=r+2}^M \mu(\gloc_i)}{2\mathcal{E}\prod_{i=2}^rW(\gloc_{i},\cdot)} d\gloc_{1:M},
\]
where $$\mathcal{R}(\gloc_{1:M})= \EE\Big( L(G_{\gloc_{1:M}},G_{\gloc_{1:M}}({\genparam_{{\gloc_{1:M}}}}))\prod_{i=1}^r \mathbb{I}({U_i\le W(\gloc_i,\gloc_{i+1}) }) \Big),$$ and where $G_{\gloc_{1:M}}$ is the subgraph with vertices having intensities respectively $\gloc_{1},\dots,\gloc_{m}$,
and $\forall i,~\genparam_{{\gloc_{i}}}\overset{iid}{\sim}m(\gloc_i,\cdot).$%

Now let write for all $n$, $\mathbb{F}_n$ the sigma-field of events invariant under joint permutations of the indexes in $[1,n]^M$. 
Then we can see that  $(\min_{\bar{\genparam}\in \Omega_{\genparam}^{\PP_n}}\frac{1}{\prod_{i=0}^{M-1}(n-i)}\sum_{I \in [|1,n-1|]^M}X_{I}({\bar{\genparam}}),\mathbb{F}_n)$ is a  reverse supermartingale. Indeed \begin{itemize}\item $\min_{\bar{\genparam}\in \Omega_{\genparam}^{\PP_n}}\frac{1}{\prod_{i=0}^{M-1}(n-i)}\sum_{I \in [|1,n-1|]^M}X_{I}({\bar{\genparam}})$ is $\mathbb{F}_n$ measurable as it is  invariant under joint permutations of the indexes in $[1,n]^M$. 
\item For all $m\ge n$ let $\hat{\genparam}_m\in \Omega_{\genparam}^{\Pi}$ such that: $$\sum_{I\in [|1,m-1|]^M}X_{I}(\hat{{\genparam}}_m)= \min_{\bar{\genparam}\in \Omega_{\genparam}^{\PP_m}}\sum_{I\in [|1,m-1|]^M}X_{I}({\bar{\genparam}})$$
Then  we get
\begin{equation*}\begin{split}&
\mathbb{E}\Big(\min_{\bar{\genparam}\in \Omega_{\genparam}^{\PP_n}}\frac{1}{n^M}\sum_{I\in [|1,n-1|]^M}X_{I}({\bar{\genparam}})|F_m)
\\&\overset{(a)}{ \le} \mathbb{E}\Big(\frac{1}{n^M}\sum_{I\in [|1,n-1|]^M}X_{I}(\hat{\genparam}_m)|F_m)
\\& \overset{(b)}{\le}  \min_{\bar{\genparam}\in \Omega_{\genparam}^{\PP_m}}\frac{1}{m^M}\sum_{I\in [|1,m-1|]^M}X_{I}({\bar{\genparam}}),
\end{split}\end{equation*}
where (a) comes from Jensen and (b) comes from a standard argument in exchangeable arrays.%
\end{itemize}
Therefore we have that:
\begin{equation*}\begin{split}
\min_{\bar{\genparam}\in \Omega_{\genparam}^{\PP}}\frac{1}{n^M}\sum_{I\in  [|1,n-1|]^M}X_{I}({\bar{\genparam}})- \EE(\min_{\bar{\genparam}\in \Omega_{\genparam}^{\PP}}\frac{1}{n^M}\sum_{I\in [|1,n-1|]^M}X_{I}({\bar{\genparam}}))\convPr0.
\end{split}\end{equation*}

\end{proof}

\section{Convergence of global parameters}
\label{sec:glob_param_conv}
We now establish the second main convergence result.
This result applies to the two stage procedure where the embeddings are learned first and the global parameters are then learned with the embeddings fixed.
The result is that the learned global parameters will converge in the ordinary statistical consistency sense.

Our proof of this guarantee requires some technical conditions.
\begin{defn}
Suppose that  $\Omega_{\globparam}$ is a compact convex set. 
A loss function ${L}$ is \defnphrase{$\epsilon$-strongly convex} in $\globparam$ if for all $ \globparam,\globparam'\in \Omega_{\globparam}$, for all $\eta\in [0,1]$, and for all $\bar{\genparam}_{\globparam},\bar{\genparam}_{\globparam'},\bar{\genparam}_{\eta\globparam'+(1-\eta)\globparam}$ 
such that 
\begin{enumerate}
\item $\embedding(\bar{\genparam}_{\globparam})=\embedding(\bar{\genparam}_{\globparam'})=\embedding(\bar{\genparam}_{\eta \globparam'+(1-\eta)\globparam})$, and
\item $\globparam(\bar{\genparam}_{\globparam})=\globparam,~\globparam(\bar{\genparam}_{\globparam'})=\globparam',~\globparam(\bar{\genparam}_{(1-\eta)\globparam+\eta \globparam'})=(1-\eta)\globparam+\eta \globparam'$
\end{enumerate}
it holds that  
\[
{L}(G_H, \bar{\genparam}_{\eta\globparam'+(1-\eta)\globparam}) \overset{\as}{<} 
\eta {L}(G_H,\bar{\genparam}_{\globparam'})+(1-\eta){L}(G_H,\bar{\genparam}_{\globparam})- \frac{1}{2} \epsilon \eta(1-\eta) \|\globparam - \globparam'\|^2_2.
\]
\end{defn}
\begin{defn}
A loss function ${L}$ is \defnphrase{uniformly continuous} if 
\[
\lim_{\globparam'\rightarrow \globparam}\Big\|\sup_{\bar{\embedding}\in \Omega_{\embedding}^{\Pi}}\big|{L}(G_H,\bar{\genparam}_{\globparam'})-{L}(G_H, \bar{\genparam}_{\globparam})\big|\Big\|_{L_1}=0.
\]
\end{defn}

We write the risk as $\hat{R}_k(\gamma,\lambda; G_n)$.%
\begin{lemma}
Suppose that there is $\epsilon>0$ such that ${L}$ is $\epsilon$-strongly convex and uniformly continuous in $\globparam$, and that  $\Omega_{\globparam}$  is a compact convex set. 
Let  $(\hat{\gamma}_n)_n\in \Omega_{\globparam}^{\mathbb{N}}$ be a sequence  of elements in $\Omega_{\globparam}$ such that, for all $n$,
\[
\min_{ \embedding\in \Omega_{\embedding}^{\Pi}} \hat{R}_k(\hat{\globparam}_n,\embedding; G_n)=\min_{\globparam \in \Omega_{\globparam}}\min_{ \embedding\in \Omega_{\embedding}^{\Pi}}\hat{R}_k(\globparam,\embedding; G_n).
\]
Then 
\[
\hat{\gamma}_n\convPr\gamma^*,
\]
where $\gamma^*=\rm{argmin}_{\gamma}{ \lim_n \mathbb{E}(\min_{ \embedding\in \Omega_{\embedding}^{\Pi}}\hat{R}_k(\globparam,\embedding;G_n))}$

\end{lemma}
\begin{remark}
This result is valid for both random-walk and $p$-sampling.
\end{remark}
\begin{proof}
Let $\hat{R}_k(\globparam; G_n)\triangleq \min_{\lambda\in \Omega_{\genparam}^{\Pi}} \hat{R}_k(\globparam,\lambda; G_n)$. 

\cref{thm:conv_of_RERM} for the random walk sampler and \cref{thm:p-samp_convergence} for $p$-sampling give the following for any 
$\globparam$:
\[
\hat{R}_k(\globparam;G_n)-\mathbb{E}(\hat{R}_k(\globparam;G_n)))\convPr0.
\]

Let $(\hat{\globparam}_n)_n\in \Omega_{\globparam}^{\mathbb{N}}$ be a sequence such that 
\[
\hat{R}_k(\hat{\globparam}_n;G_n)=\min_{\globparam\in \Omega_{\globparam}}\hat{R}_k(\globparam;G_n).
\]
Since $(\hat{\globparam}_n)_n$ is a sequence in the compact set $\Omega_{\globparam}$ there is a function $\phi:\mathbb{N}\rightarrow \mathbb{N}$ and $\tilde{\globparam}$ such that $\hat{\globparam}_{\phi(n)}\xrightarrow{d} \tilde{\globparam}.$ 
But as $\Omega_{\globparam}$ is compact, an easy consequence of the covering lemma gives that: %
$$\sup_{\globparam\in \Omega_{\globparam}}\Big|\hat{R}_k(\globparam;G_n)-f(\globparam)\Big|\convPr0,$$
where $f:\globparam\rightarrow \lim_{n}\mathbb{E}(\hat{R}_k(\globparam;G_n)).$
Therefore  we have that $$|\hat{R}_k(\hat{\globparam}_{\phi(n)},G_{\phi(n)})-f(\hat{\globparam}_{\phi(n)})|\convPr 0.$$

But using the expressions \cref{rw_samp} and \cref{p_sam} derived in the proof 
 of respectively \cref{thm:conv_of_RERM} and \cref{thm:p-samp_convergence} and the 
\NA{$\epsilon$-strongly convex assumption on $L$}
we have that $f$ is continuous and  is strictly convex, and hence has a unique minimizer. %

Therefore $\tilde{\globparam}$ must be deterministic  equal to $\globparam^*\triangleq \argmin_{\globparam} f(\globparam)$. Indeed suppose by contradiction that it is not the case then there is $\eta>0$ such that  
\[
\Pr(\hat{R}_k(\hat{\globparam}_{\phi(s)},G_{\phi(s)})- \hat{R}_k(\globparam^*,G_{\phi(s)})>\eta)>\eta,
\] which is a contradiction of the definition of $(\hat{\globparam}_n)_n$.
 Therefore we have successfully proven that $\tilde\globparam=\globparam^*$.

And we have proved that $\hat{\globparam}_n\convPr \globparam^*.$

\end{proof}

\section{Stability of embeddings}
\label{sec:stability-of-embeddings}
\begin{theorem}
Suppose the conditions of \cref{thm:emp_risk_conv} (i.e., the form of $\samp$, that $\dataset_n$ is generated by a graphon process,
and that parameter settings are markings of the latent Poisson process). 
Suppose that the loss function is twice differentiable and the Hessian of the empirical risk is bounded.
Let $\hat{\embedding}_{n+1}\restrict_n$ denote the restriction of the embeddings $\hat{\embedding}_{n+1}$
to the vertices present in $G_n$.
Then
$
\hat{\embedding}_{n} - \hat{\embedding}_{n+1}\restrict_n \to 0
$
in probability, as $n\to\infty$.
\end{theorem}
\begin{proof}
For notational simplicity, we consider the case with no global parameters and note that the same proof works if global parameters are included.

First, by a Taylor expansion about $\hat{\embedding}_n$,
\[
\hat{R}_k(\hat{\embedding}_{n+1} \restrict_n; \dataset_{n}) 
= \hat{R}_k(\hat{\embedding}_{n}; \dataset_n) + 0 + \nicefrac{1}{2} (\hat{\embedding}_{n} - \hat{\embedding}_{n+1} \restrict_n)^T H_n (\hat{\embedding}_{n} - \hat{\embedding}_{n+1}\restrict_n),
\]
where $H_n$ is the Hessian evaluated at an appropriate point.
Then, to prove the result it suffices to show that $\hat{R}_k(\hat{\embedding}_{n+1} \restrict_n; \dataset_{n}) - \hat{R}_k(\hat{\embedding}_{n}; \dataset_n) \convPr 0$ as $n\to\infty$.

To that end, we first show $\hat{R}_k(\hat{\embedding}_{n+1} \restrict_n; \dataset_{n}) \approx \hat{R}_k(\hat{\embedding}_{n+1}; \dataset_{n+1})$.
By \citep[Prop.~30]{Borgs:Chayes:Cohn:Holden:2016}, $E_n / n^2 \to \mathcal{E}$ a.s.\ as $n\to\infty$. 
Then, the expected number of edges selected by $\samp(\dataset_{n+1},k)$ that do not belong to $\dataset_n$ is:
\begin{align}
\label{eq:expected-edges-not-chosen}
k (1 - \EE[\edges(\dataset_n) / \edges(\dataset_{n+1}) \given \dataset_{n+1}] = o(1) \as 
\end{align}
We expand $\hat{R}_k(\hat{\embedding}_{n+1}; \dataset_{n+1})$ as:
\begin{align}
\label{eq:expanded-risk}
\begin{split}
\EE[L(\samp(\dataset_{n+1},k); \hat{\embedding}_{n+1}) \given \dataset_{n+1}] &= 
\EE[L(\samp(\dataset_{n},k); \hat{\embedding}_{n+1} \restrict_n) \given \dataset_{n}] \Pr(\samp(\dataset_{n+1},k) \subset \dataset_n \given \dataset_{n+1} ) \\
&  + \EE[L(\samp(\dataset_{n+1},k); \hat{\embedding}_{n+1}) \given \dataset_{n+1}] \Pr(\samp(\dataset_{n+1},k) \subsetneq \dataset_n \given \dataset_{n+1}).
\end{split}
\end{align}
By Markov's inequality and \cref{eq:expected-edges-not-chosen}, 
\[
\Pr(\samp(\dataset_{n+1},k) \subsetneq \dataset_n \given \dataset_{n+1}) \convPr 0,
\] 
as $n \to \infty$.
By \cref{thm:emp_risk_conv}, $\EE[L(\samp(\dataset_{n+1},k); \hat{\embedding}_{n+1}) \given \dataset_{n+1}]$ converges to a constant in probability,
so the second term of \cref{eq:expanded-risk} converges to $0$ in probability.
Hence, 
\begin{align}
\label{eq:conv-1}
\hat{R}_k(\hat{\embedding}_{n+1} \restrict_n; \dataset_{n}) - \hat{R}_k(\hat{\embedding}_{n+1}; \dataset_{n+1}) \convPr 0,
\end{align} 
as $n\to\infty$.

By \cref{thm:emp_risk_conv}, 
\begin{align}
\label{eq:conv-2}
\hat{R}_k(\hat{\embedding}_{n}; \dataset_{n}) - \hat{R}_k(\hat{\embedding}_{n+1}; \dataset_{n+1}) \convPr 0,
\end{align}
as $n \to \infty$.
The proof is completed by combining \cref{eq:conv-1,eq:conv-2}.
\end{proof}

\printbibliography[heading=subbibliography]
\end{refsection}

\end{document}